 \documentclass[final,1p,times]{elsarticle}

\usepackage{amssymb}
\usepackage{amsmath}
\usepackage{amsthm}
\usepackage{array}
\usepackage{multirow}
\usepackage{mdwmath}
\usepackage{mdwtab}
\usepackage[tight,footnotesize]{subfigure}
\usepackage{amsfonts}
\usepackage{algorithm}
\usepackage{algorithmic}
\usepackage{hyperref}

\newtheorem{definition}{Definition}
\newtheorem{theorem}{Theorem}
\newtheorem{lemma}[theorem]{Lemma}

\begin{document}

\begin{frontmatter}



\title{Structured Low-Rank Matrix Factorization with Missing and Grossly Corrupted Observations}

\author[label1]{Fanhua Shang}
\cortext[]{Corresponding author, e-mail: fhshang@cse.cuhk.edu.hk}
\author[label2]{Yuanyuan Liu}
\author[label3]{Hanghang Tong}
\author[label1]{James Cheng}
\author[label2]{Hong Cheng}

\address[label1]{Department of Computer Science and Engineering, The Chinese University of Hong Kong}
\address[label2]{Department of Systems Engineering and Engineering Management, The Chinese University of Hong Kong}
\address[label3]{School of computing informatics and decision systems engineering, Arizona State University}

\begin{abstract}
Recovering low-rank and sparse matrices from incomplete or corrupted observations is an important problem in machine learning, statistics, bioinformatics, computer vision, as well as signal and image processing. In theory, this problem can be solved by the natural convex joint/mixed relaxations (i.e., $l_{1}$-norm and trace norm) under certain conditions. However, all current provable algorithms suffer from superlinear per-iteration cost, which severely limits their applicability to large-scale problems. In this paper, we propose a scalable, provable structured low-rank matrix factorization method to recover low-rank and sparse matrices from missing and grossly corrupted data, i.e., robust matrix completion (RMC) problems, or incomplete and grossly corrupted measurements, i.e., compressive principal component pursuit (CPCP) problems. Specifically, we first present two small-scale matrix trace norm regularized bilinear structured factorization models for RMC and CPCP problems, in which repetitively calculating SVD of a large-scale matrix is replaced by updating two much smaller factor matrices. Then, we apply the alternating direction method of multipliers (ADMM) to efficiently solve the RMC problems. Finally, we provide the convergence analysis of our algorithm, and extend it to address general CPCP problems. Experimental results verified both the efficiency and effectiveness of our method compared with the state-of-the-art methods.
\end{abstract}

\begin{keyword}
Compressive principal component pursuit, Robust matrix completion, Robust principal component analysis, Low-rank matrix recovery and completion
\end{keyword}

\end{frontmatter}

\section{Introduction}
In recent years, recovering low-rank and sparse matrices from severely incomplete or even corrupted observations has received broad attention in many different fields, such as statistics \cite{wright:cpcp, agarwal:nmd, chen:lrmr}, bioinformatics \cite{otazo:lrsd}, machine learning \cite{wright:rpca, waters:rcs, tao:lrsc, liu:nnr}, computer vision \cite{candes:rpca, zhou:gb, zheng:rma, cabral:nnbf, shang:rpca}, signal and image processing \cite{ma:lrmc, peng:rasl, li:lrml, liu:mbf, feng:lr}. In those areas, the data to be analyzed often have high dimensionality, which brings great challenges to data analysis, such as digital photographs, surveillance videos, text and web documents. Fortunately, the high-dimensional data are observed to have low intrinsic dimension, which is often much smaller than the dimension of the ambient space \cite{liu:lrr}.

For the high-dimensional data, principal component analysis (PCA) is one of the most popular analysis tools to recover a low-rank structure of the data mainly because it is simple to implement, can be solved efficiently, and is effective in many real-world applications such as face recognition and text clustering. However, one of the main challenges faced by PCA is that the observed data is often contaminated by outliers and missing values \cite{favaro:ssc}, or is a small set of linear measurements \cite{wright:cpcp}. To address these issues, many compressive sensing or rank minimization based techniques and methods have been proposed, such as robust PCA \cite{wright:rpca, xu:rpca, shang:rpca} (RPCA, also called PCP in \cite{candes:rpca} and low-rank and sparse matrix decomposition in \cite{tao:lrsc, yuan:slmd}, LRSD) and low-rank matrix completion \cite{candes:emc, chen:lrmr}.

In many applications, we have to recover a matrix from only a small number of observed entries, for example collaborative filtering for recommender systems. This problem is often called matrix completion, where missing entries or outliers are presented at arbitrary location in the measurement matrix. Matrix completion has been used in a wide range of problems such as collaborative filtering \cite{candes:emc, chen:lrmr}, structure-from-motion \cite{eriksson:lrma, zheng:rma}, click prediction \cite{yu:cp}, tag recommendation \cite{wang:at}, and face reconstruction \cite{meng:cwm}. In some other applications, we would like to recover low-rank and sparse matrices from corrupted data. For example, the face images of a person may be corrupted by glasses or shadows \cite{lee:fr}. The classical PCA cannot address the issue as its least-squares fitting is sensitive to these gross outliers. Recovering a low-rank matrix in the presence of outliers has been extensively studied, which is often called RPCA, PCP or LRSD. The RPCA problem has been successfully applied in many important applications, such as latten semantic indexing \cite{min:pcp}, video surveillance \cite{wright:rpca, candes:rpca}, and image alignment \cite{peng:rasl}. In some more general applications, we also have to simultaneously recover both low-rank and sparse matrices from small sets of linear measurements, which is called compressive principal component pursuit (CPCP) in \cite{wright:cpcp}.

In principle, those problems mentioned above can be exactly solved with high probability under mild assumptions via a hybrid convex program involving both the $l_{1}$-norm and the trace norm (also called the nuclear norm) minimization. In recent years, many new techniques and algorithms \cite{candes:emc, chen:lrmr, candes:rpca, wright:rpca, xu:rpca, wright:cpcp} for low-rank matrix recovery and completion have been proposed, and the theoretical guarantees have been derived in \cite{candes:emc, candes:rpca, wright:cpcp}. However, those provable algorithms all exploit a closed-form expression for the proximal operator of the trace norm, which involves the singular value decomposition (SVD). Hence, they all have high computational cost and are even not applicable for solving large-scale problems.

To address this issue, we propose a scalable robust bilinear structured factorization (RBF) method to recover low-rank and sparse matrices from incomplete, corrupted data or a small set of linear measurements, which is formulated as follows:
\begin{equation}\label{e1.1}
\min_{L,\,S} f(L,S)+\lambda\|L\|_{\ast},\quad \textup{s.t.},\,\mathcal{A}(L+S)=y,
\end{equation}
where $\lambda\geq 0$ is a regularization parameter, $\|L\|_{\ast}$ is the trace norm of a low-rank matrix $L\in\mathbb{R}^{m\times n}$, i.e., the sum of its singular values, $S\in\mathbb{R}^{m\times n}$ is a sparse error matrix, $y\in\mathbb{R}^{p}$ is the given linear measurements, $\mathcal{A}(\cdot)$ is an underdetermined linear operator such as the linear projection operator $\mathcal{P}_{\Omega}$, and $f(\cdot)$ denotes the loss function such as the $l_{2}$-norm loss or the $l_{1}$-norm loss functions.

\begin{figure}[t]
\centering
\includegraphics [width=0.40\linewidth] {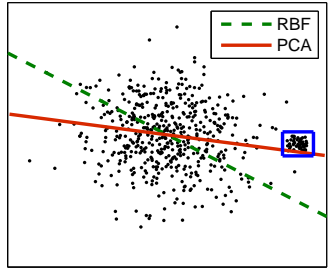}
\caption{Principal directions learned by PCA and RBF on the toy data set with outliers, which are in a blue rectangle.}
\label{fig_1}
\end{figure}

Unlike existing robust low-rank matrix factorization approaches, our method not only takes into account the fact that the observation is contaminated by additive outliers (Fig.\ 1 shows an example) or missing data, i.e., robust matrix completion \cite{chen:lrmr, li:cs} (RMC, also called RPCA plus matrix completion) problems, but can also identify both low-rank and sparse noisy components from incomplete and grossly corrupted measurements, i.e., CPCP problems. We also present a robust bilateral factorization framework for both RMC and CPCP problems, in which repetitively calculating SVD of a large matrix in \cite{candes:emc, candes:rpca, wright:cpcp} is replaced by updating two much smaller factor matrices. We verify with convincing experimental results both the efficiency and effectiveness of our RBF method.

The main contributions of this paper are summarized as follows:
\begin{enumerate}
\item We propose a scalable structured RBF framework to simultaneously recover both low-rank and sparse matrices for both RMC and CPCP problems. By imposing the orthogonality constraint, we convert the original RMC and CPCP models into two smaller-scale matrix trace norm regularized problems, respectively.
\item By the fact that the optimal solution $S_{\Omega^{C}}=0$, i.e., the values of $S$ at unobserved locations are zero, we reformulate the proposed RMC problem by replacing the linear projection operator constraint with a simple equality one.
\item Moreover, we propose an efficient alternating direction method of multipliers (ADMM) to solve our RMC problems, and then extend it to address CPCP problems with a linearization technique.
\item Finally, we theoretically analyze the suboptimality of the solution produced by our algorithm.
\end{enumerate}

The remainder of this paper is organized as follows. We review background and related work in Section 2. In Section 3, we propose two scalable trace norm regularized RBF models for RMC and CPCP problems. We develop an efficient ADMM algorithm for solving RMC problems and then extend it to solve CPCP problems in Section 4. We provide the theoretical analysis of our algorithm in Section 5. We report empirical results in Section 6, and conclude this paper in Section 7.

\section{BACKGROUND}
A low-rank structured matrix $L\in\mathbb{R}^{m\times n}$ and a sparse one $S\in\mathbb{R}^{m\times n}$ can be recovered from highly corrupted measurements $y=\mathcal{P}_{Q}(D)\in\mathbb{R}^{p}$ via the following CPCP model,
\begin{equation}\label{e2.2}
\min_{L,\,S} \|S\|_{1}+\lambda\|L\|_{\ast},\quad \textup{s.t.},\,\mathcal{P}_{Q}(D=L_{0}+S_{0})=\mathcal{P}_{Q}(L+S),
\end{equation}
where $\|S\|_{1}$ denotes the $l_{1}$-norm of $S$, i.e., $\|S\|_{1}=\Sigma_{ij}|s_{ij}|$, $Q\subseteq\mathbb{R}^{m\times n}$ is a linear subspace, and $\mathcal{P}_{Q}$ is the projection operator onto that subspace. When $\mathcal{P}_{Q}$$=\mathcal{P}_{\Omega}$, the model (2) is the robust matrix completion (RMC) problem, where $\Omega$ is the index set of observed entries. Wright et al., \cite{wright:cpcp} proved the following result.
\begin{theorem}
Let $L_{0}$, $S_{0}\in\mathbb{R}^{m\times n}$, with $m\geq n$, and suppose that $L_{0}\neq$\emph{\textbf{0}} is a $\mu$-incoherent matrix of rank $r$,
\begin{displaymath}
r \leq\frac{c_{1}n}{\mu\log^{2}m},
\end{displaymath}
and sign($S_{0}$) is i.i.d.\ Bernoulli\texttt{-}Rademacher with non-zero probability $\rho<c_{2}$. Let $Q\subset\mathbb{R}^{m\times n}$ be a random subspace of dimension
\begin{displaymath}
\textup{dim}(Q) \geq C_{1}(\rho mn+mr)\log^{2} m,
\end{displaymath}
distributed according to the Haar measure, probabilistically independent of sign($S_{0}$). Then the minimizer to the problem (2) with $\lambda=\sqrt{m}$ is unique and equal to $(L_{0},S_{0})$ with probability at least $1-C_{2}m^{-9}$, where $c_{1}$, $c_{2}$, $C_{1}$ and $C_{2}$ are positive numerical constants.
\end{theorem}
This theorem states that a commensurately small number of measurements are sufficient to accurately recover the low-rank and sparse matrices with high probability. Indeed, if $Q$ is the entire space, the model (2) degenerates to the following RPCA problem \cite{wright:rpca, candes:rpca}
\begin{equation}\label{e2.3}
\min_{L,\,S} \|S\|_{1}+\lambda\|L\|_{\ast},\quad \textup{s.t.},\,D=L+S,
\end{equation}
where $D$ denotes the given observations. Several algorithms have been developed to solve the convex optimization problem (\ref{e2.3}), such as IALM \cite{lin:ladmm} and LRSD \cite{yuan:slmd}. Although both models (2) and (3) are convex optimization problems, and their algorithms converge to the globally optimal solution, they involve SVD at each iteration and suffer from a high computational cost of $O(mn^{2})$. While there have been many efforts towards fast SVD computation such as partial SVD \cite{lin:ladmm} and approximate SVD \cite{ma:fpca}, the performance of these existing methods is still unsatisfactory for many real applications. To address this problem, we propose a scalable, provable robust bilinear factorization method with missing and grossly corrupted observations.

\section{OUR RBF FRAMEWORK}
Matrix factorization is one of the most useful tools in scientific computing and high dimensional data analysis, such as the QR decomposition, the LU decomposition, SVD, and NMF. In this paper, robust bilinear factorization (RBF) aims to find two smaller low-rank matrices $U\in\mathbb{R}^{m\times d}$ $(U^{T}U=I)$ and $V\in\mathbb{R}^{n\times d}$ whose product is equal to the desired matrix of low-rank, $L\in\mathbb{R}^{m\times n}$,
\begin{displaymath}
L=UV^{T},
\end{displaymath}
where $d$ is an upper bound for the rank of $L$, i.e., $d\geq r=\textrm{rank}(L)$.

\subsection{RMC Model}
Suppose that the observed matrix $D$ is corrupted by outliers and missing data, the RMC problem is given by
\begin{equation} \label{e3.4}
\min_{L,\,S} \|S\|_{1}+\lambda\|L\|_{\ast},\quad \textup{s.t.},\,\mathcal{P}_{\Omega}(D)=\mathcal{P}_{\Omega}(L+S).
\end{equation}
From the optimization problem (4), we easily find the optimal solution $S_{\Omega^{C}}=0$, where $\Omega^{C}$ is the complement of $\Omega$, i.e., the index set of unobserved entries. Consequently, we have the following lemma.

\begin{lemma}
The RMC model (\ref{e3.4}) with the operator $\mathcal{P}_{\Omega}$ is equivalent to the following problem
\begin{equation} \label{e3.5}
\min_{L,\,S} \|\mathcal{P}_{\Omega}(S)\|_{1}+\lambda\|L\|_{\ast},\quad \textup{s.t.},\,\mathcal{P}_{\Omega}(D)=\mathcal{P}_{\Omega}(L+S)\;\,\textup{and}\;\,\mathcal{P}_{\Omega^{C}}(S)=\emph{\textbf{0}}.
\end{equation}
\end{lemma}

The proof of this lemma can be found in APPENDIX A. From the incomplete and corrupted matrix $D$, our RBF model is to find two smaller matrices, whose product approximates $L$, can be formulated as follows:
\begin{equation} \label{e3.6}
\begin{split}
&\min_{U,\,V,\,S}\,\|\mathcal{P}_{\Omega}(S)\|_{1}+\lambda\|UV^{T}\|_{*},\\
&\textup{s.t.},\,\mathcal{P}_{\Omega}(D)=\mathcal{P}_{\Omega}(UV^{T}+S).
\end{split}
\end{equation}

\begin{lemma}
Let $U$ and $V$ be two matrices of compatible dimensions, where $U$ has orthogonal columns, i.e., $U^{T}U=I$, then we have $\|UV^{T}\|_{*}=\|V\|_{*}$.
\end{lemma}

The proof of this lemma can be found in APPENDIX B. By imposing $U^{T}U=I$ and substituting $\|UV^{T}\|_{*}=\|V\|_{*}$ into (\ref{e3.6}), we arrive at a much smaller-scale matrix trace norm minimization problem
\begin{equation} \label{e3.7}
\begin{split}
&\min_{U,\,V,\,S}\,\|\mathcal{P}_{\Omega}(S)\|_{1}+\lambda\|V\|_{*},\\
&\textup{s.t.},\,\mathcal{P}_{\Omega}(D)=\mathcal{P}_{\Omega}(UV^{T}+S),\,U^{T}U=I.
\end{split}
\end{equation}

\begin{theorem}
Suppose $(L^{\ast},\,S^{\ast})$ is a solution of the problem (\ref{e3.5}) with $\textrm{rank}(L^{\ast})=r$, then there exists the solution $U_{k}\in \mathbb{R}^{m\times d}$, $V_{k}\in\mathbb{R}^{n\times d}$ and $S_{k}\in\mathbb{R}^{m\times n}$ to the problem (\ref{e3.7}) with $d\geq r$ and $S_{\Omega^{C}}=0$, $(U_{k}V^{T}_{k},\,S_{k})$ is also a solution to the problem (\ref{e3.5}).
\end{theorem}
The proof of this theorem can be found in APPENDIX C.

\subsection{CPCP Model}
From a small set of linear measurements $y\in\mathbb{R}^{p}$, the CPCP problem is to recover low-rank and sparse matrices as follows,
\begin{equation}\label{e3.8}
\begin{split}
&\min_{U,\,V,\,S} \|S\|_{1}+\lambda\|V\|_{\ast},\\
&\;\textup{s.t.},\,\mathcal{P}_{Q}(D)=\mathcal{P}_{Q}(UV^{T}+S).
\end{split}
\end{equation}

\begin{theorem}
Suppose $(L^{\ast},\,S^{\ast})$ is a solution of the problem (\ref{e2.2}) with $\textrm{rank}(L^{\ast})=r$, then there exists the solution $U_{k}\in \mathbb{R}^{m\times d}$, $V_{k}\in \mathbb{R}^{n\times d}$ and $S_{k}\in\mathbb{R}^{m\times n}$ to the problem (\ref{e3.8}) with $d\geq r$, $(U_{k}V^{T}_{k},\,S_{k})$ is also a solution to the problem (\ref{e2.2}).
\end{theorem}

We omit the proof of this theorem since it is very similar to that of Theorem 4. In the following, we will discuss how to solve the models (\ref{e3.7}) and (\ref{e3.8}). It is worth noting that the RPCA problem can be viewed as a special case of the RMC problem (\ref{e3.7}) when all entries of the corrupted matrix are directly observed. In the next section, we will mainly develop an efficient alternating direction method of multipliers (ADMM) solver for solving the non-convex problem (\ref{e3.7}). It is also worth noting that although our algorithm will produce different estimations of $U$ and $V$, the estimation of $UV^{T}$ is stable as guaranteed by Theorems 4 and 5, and the convexity of the problems (\ref{e2.2}) and (\ref{e3.4}).

\subsection{Connections to Existing Approaches}
According to the discussion above, it is clear that our RBF method is a scalable method for both RMC and CPCP problems. Compared with convex algorithms such as common RPCA \cite{candes:rpca} and CPCP \cite{wright:cpcp} methods, which have a computational complexity of $O(mn^2)$ and are impractical for solving relatively large-scale problems, our RBF method has a linear complexity and scales well to handle large-scale problems.

To understand better the superiority of our RBF method, we compare and relate RBF with several popular robust low-rank matrix factorization methods. It is clear that the model in \cite{shen:alad, zhou:gb, meng:cwm} is a special case of our trace norm regularized model (7), while $\lambda=0$. Moreover, the models used in \cite{zheng:rma, cabral:nnbf} focus only on the desired low-rank matrix. In this sense, they can be viewed as special cases of our model (7). The other major difference is that SVD is used in \cite{zheng:rma}, while QR factorizations are used in this paper. The use of QR factorizations also makes the update operation highly scalable on modern parallel architectures \cite{avron:ssgd}. Regarding the complexity, it is clear that both schemes have the similar theory computational complexity. However, from the experimental results in Section 6, we can see that our algorithm usually runs much faster, but more accurate than the methods in \cite{zheng:rma, cabral:nnbf}. The following bilinear spectral regularized matrix factorization formulation in \cite{cabral:nnbf} is one of the most similar models to our model (7),
\begin{equation*}
\min_{L,\,U,\,V}\,\|W\odot(D-L)\|_{1}+\frac{\lambda}{2}(\|U\|^{2}_{F}+\|V\|^{2}_{F}),\quad\textup{s.t.},\,L=UV^{T},
\end{equation*}
where $\odot$ denotes the Hadamard product and $W\in\mathbb{R}^{m\times n}$ is a weight matrix that can be used to denote missing data (i.e., $w_{ij}=0$).

\section{OPTIMIZATION ALGORITHM}
In this section, we propose an efficient alternating direction method of multipliers (ADMM) for solving the RMC problem (\ref{e3.7}), and then extend it for solving the CPCP problem (8). We provide the convergence analysis of our algorithm in Section 5.

\subsection{Formulation}
Recently, it has been shown in the literature \cite{boyd:admm, wen:nsor} that ADMM is very efficient for some convex or non-convex programming problems from various applications. We also refer to a recent survey \cite{boyd:admm} for some recently exploited applications of ADMM. For efficiently solving the RMC problem (7), we can assume without loss of generality that the unknown entries of $D$ are simply set as zeros, i.e., $D_{\Omega^{C}}$=0, and $S_{\Omega^{C}}$ may be any values such that $\mathcal{P}_{\Omega^{C}}(D)=\mathcal{P}_{\Omega^{C}}(UV^{T})+\mathcal{P}_{\Omega^{C}}(S)$. Therefore, the constraint with the operator $\mathcal{P}_{\Omega}$ in (7) is simplified into $D=UV^{T}+S$. Hence, we introduce the constraint $D=UV^{T}+S$ into (7), and obtain the following equivalent form:
\begin{equation} \label{e4.9}
\begin{split}
&\min_{U,\,V,\,S}\,\|\mathcal{P}_{\Omega}(S)\|_{1}+\lambda\|V\|_{*},\\
&\textup{s.t.},\,D=UV^{T}+S,\,U^{T}U=I.
\end{split}
\end{equation}
The partial augmented Lagrangian function of (9) is
\begin{equation} \label{e4.10}
\begin{split}
&\;\mathcal{L}_{\alpha}(U,V,S,Y)=\lambda\|V\|_{*}+\|\mathcal{P}_{\Omega}(S)\|_{1}\\
+&\langle Y,D-S-UV^{T}\rangle+\frac{\alpha}{2}\|D-S-UV^{T}\|^{2}_{F},
\end{split}
\end{equation}
where $Y\in\mathbb{R}^{m\times n}$ is a matrix of Lagrange multipliers, $\alpha>0$ is a penalty parameter, and $\langle M, N\rangle$ denotes the inner product between matrices $M$ and $N$ of equal sizes, i.e., $\langle M,N\rangle=\Sigma_{i,j}M_{ij}N_{ij}$.

\subsection{Robust Bilateral Factorization Scheme}
We will derive our scheme for solving the following subproblems with respect to $U$, $V$ and $S$, respectively,
\begin{equation} \label{e4.11}
\begin{split}
U_{k+1}=&\mathop{\arg\min}_{U\in\mathbb{R}^{m\times d}}\mathcal {L}_{\alpha_{k}}(U,V_{k},S_{k},Y_{k}),\;\;\;\;\;\;\;\\
&\textup{s.t.},\,U^{T}U=I,
\end{split}
\end{equation}
\begin{equation} \label{e4.12}
V_{k+1}=\mathop{\arg\min}_{V\in\mathbb{R}^{n\times d}}\mathcal {L}_{\alpha_{k}}(U_{k+1},V,S_{k},Y_{k}),\;\;\;
\end{equation}
\begin{equation} \label{e4.13}
S_{k+1}=\mathop{\arg\min}_{S\in\mathbb{R}^{m\times n}}\mathcal {L}_{\alpha_{k}}(U_{k+1},V_{k+1},S,Y_{k}).
\end{equation}

\subsubsection{Updating U}
Fixing $V$ and $S$ at their latest values, and by removing the terms that do not depend on $U$ and adding some proper terms that do not depend on $U$, the problem (11) with respect to $U$ is reformulated as follows:
\begin{equation} \label{e4.14}
\min_{U}\,\|UV^{T}_{k}-P_{k}\|^{2}_{F},\quad \textup{s.t.},\,U^{T}U=I,
\end{equation}
where $P_{k}=D-S_{k}+Y_{k}/\alpha_{k}$. In fact, the optimal solution can be given by the SVD of the matrix $P_{k}V_{k}$ as in \cite{nick:mpp}. To further speed-up the calculation, we introduce the idea in \cite{wen:nsor} that uses a QR decomposition instead of SVD. The resulting iteration step is formulated as follows:
\begin{equation} \label{e4.15}
U_{k+1}=Q,\quad \textrm{QR}(P_{k}V_{k})=QR,
\end{equation}
where $U_{k+1}$ is an orthogonal basis for the range space $\mathcal{R}(P_{k}V_{k})$, i.e., $\mathcal{R}(U_{k+1})=\mathcal{R}(P_{k}V_{k})$. Although $U_{k+1}$ in (15) is not an optimal solution to (14), our iterative scheme and the one in \cite{liu:as} are equivalent to solve (14) and (16), and their equivalent analysis is provided in Section 5. Moreover, the use of QR factorizations also makes our update scheme highly scalable on modern parallel architectures \cite{avron:ssgd}.

\subsubsection{Updating V}
Fixing $U$ and $S$, the optimization problem (12) with respect to $V$ can be rewritten as:
\begin{equation} \label{e4.16}
\min_{V}\,\frac{\alpha_{k}}{2}\|U_{k+1}V^{T}-P_{k}\|^{2}_{F}+\lambda\|V\|_{*}.
\end{equation}
To solve the convex problem (16), we first introduce the following definition \cite{cai:svt}.

\begin{definition}
For any given matrix $M\in\mathbb{R}^{n\times d}$ whose rank is $r$, and $\mu\geq0$, the singular value thresholding (SVT) operator is defined as follows:
\begin{equation*}
\textup{SVT}_{\mu}(M)=\overline{U}\textup{diag}(\max\{\sigma-\mu,0\})\overline{V}^{T},
\end{equation*}
where $\max\{\cdot,\cdot\}$ should be understood element-wise, $\overline{U}\in\mathbb{R}^{n\times r}$, $\overline{V}\in\mathbb{R}^{d\times r}$ and $\sigma=(\sigma_{1},\ldots,\sigma_{r})^{T}\in\mathbb{R}^{r\times 1}$ are obtained by SVD of $M$, i.e., $M=\overline{U}\,\textup{diag}(\sigma)\,\overline{V}^{T}$.
\end{definition}

\begin{theorem}
The trace norm minimization problem (16) has a closed-form solution given by:
\begin{equation} \label{e4.19}
V_{k+1}=\emph{SVT}_{\lambda/\alpha_{k}}(P^{T}_{k}U_{k+1}).
\end{equation}
\end{theorem}

\begin{proof}
The first-order optimality condition for (16) is given by
\begin{equation*}
0\in \lambda \partial\|V\|_{*}+\alpha_{k}(VU^{T}_{k+1}-P^{T}_{k})U_{k+1},
\end{equation*}
where $\partial\|\cdot\|_{*}$ is the set of subgradients of the trace norm. Since $U^{T}_{k+1}U_{k+1}=I$, the optimality condition for (16) is rewritten as follows:
\begin{equation}
0\in \lambda \partial\|V\|_{*}+\alpha_{k}(V-P^{T}_{k}U_{k+1}).
\end{equation}
(18) is also the optimality condition for the following convex problem,
\begin{equation}
\min_{V}\,\frac{\alpha_{k}}{2}\|V-P^{T}_{k}U_{k+1}\|^{2}_{F}+\lambda\|V\|_{*}.
\end{equation}
By Theorem 2.1 in \cite{cai:svt}, then the optimal solution of (19) is given by (17).
\end{proof}

\subsubsection{Updating S}
Fixing $U$ and $V$, we can update $S$ by solving
\begin{equation} \label{e4.20}
\min_{S}\,\|\mathcal{P}_{\Omega}(S)\|_{1}+\frac{\alpha_{k}}{2}\|S+U_{k+1}V^{T}_{k+1}-D-Y_{k}/\alpha_{k}\|^{2}_{F}.\\
\end{equation}
For solving the problem (20), we introduce the following soft-thresholding operator $\mathcal{S}_{\tau}$:
\begin{displaymath}
\mathcal{S}_{\tau}(A_{ij}):=\left\{
\begin{array}{lr}
{A_{ij}-\tau},&{A_{ij}>\tau},\\
{A_{ij}+\tau},&{A_{ij}<-\tau},\\
0,&\textup{otherwise}.
\end{array}
\right.
\end{displaymath}

Then the optimal solution $S_{k+1}$ can be obtained by solving the following two subproblems with respect to $S_{\Omega}$ and $S_{\Omega^{C}}$, respectively. The optimization problem with respect to $S_{\Omega}$ is first formulated as follows:
\begin{equation} \label{e4.21}
\min_{S_{\Omega}}\,\frac{\alpha_{k}}{2}\|\mathcal{P}_{\Omega}(S+U_{k+1}V^{T}_{k+1}-D-Y_{k}/\alpha_{k})\|^{2}_{F}+\|\mathcal{P}_{\Omega}(S)\|_{1}.
\end{equation}
By the operator $\mathcal{S}_{\tau}$ and letting $\tau=1/\alpha_{k}$, the closed-form solution to the problem (21) is given by
\begin{equation} \label{e4.22}
(S_{k+1})_{\Omega}=\mathcal{S}_{\tau}((D-U_{k+1}V^{T}_{k+1}+Y_{k}/\alpha_{k})_{\Omega}).
\end{equation}
Moreover, the subproblem with respect to $S_{\Omega^{C}}$ is formulated as follows:
\begin{equation} \label{e4.23}
\min_{S_{\Omega^{C}}}\,\|\mathcal{P}_{\Omega^{C}}(S+U_{k+1}V^{T}_{k+1}-D-Y_{k}/\alpha_{k})\|^{2}_{F}.\\
\end{equation}
We can easily obtain the closed-form solution by zeroing the gradient of the cost function (23) with respect to $S_{\Omega^{C}}$, i.e.,
\begin{equation} \label{e4.24}
(S_{k+1})_{\Omega^{C}}=(D-U_{k+1}V^{T}_{k+1}+Y_{k}/\alpha_{k})_{\Omega^{C}}.
\end{equation}

Summarizing the analysis above, we obtain an ADMM scheme to solve the RMC problem (7), as outlined in \textbf{Algorithm 1}. Our algorithm is essentially a Gauss-Seidel-type scheme of ADMM, and the update strategy of the Jacobi version of ADMM is easily implemented, well suited for parallel and distributed computing and hence is particularly attractive for solving large-scale problems \cite{shang:hotd}. In addition, $S_{\Omega^{C}}$ should be set to 0 for the expected output $S$. This algorithm can also be accelerated by adaptively changing $\alpha$. An efficient strategy \cite{lin:ladmm} is to let $\alpha=\alpha_{0}$ (the initialization in Algorithm 1) and increase $\alpha_{k}$ iteratively by $\alpha_{k+1}=\rho\alpha_{k}$, where $\rho\in (1.0,1.1]$ in general and $\alpha_{0}$ is a small constant. Furthermore, $U_{0}$ is initialized with $\textup{eye}(m, d):=\begin{bmatrix}\begin{smallmatrix} I_{d\times d}\\ \textbf{0}_{(m-d)\times d}\end{smallmatrix}\end{bmatrix}$. Algorithm 1 is easily used to solve the RPCA problem (3), where all entries of the corrupted matrix are directly observed. Moreover, we introduce an adaptive rank adjusting strategy for our algorithm in Section 4.4.

\begin{algorithm}[t]
\caption{Solving RMC problem (7) via ADMM.}
\label{alg:Framwork1}
\renewcommand{\algorithmicrequire}{\textbf{Input:}}
\renewcommand{\algorithmicensure}{\textbf{Initialize:}}
\renewcommand{\algorithmicoutput}{\textbf{Output:}}
\begin{algorithmic}[1]
\REQUIRE $\mathcal{P}_{\Omega}(D)$, $\lambda$ and $\varepsilon$.
\OUTPUT $U$, $V$ and $S$, where $S_{\Omega^{C}}$ is set to 0.\\
\ENSURE $U_{0}=\textup{eye}(m, d)$, $V_{0}=$\textbf{0}, $Y_{0}=$\textbf{0}, $\alpha_{0}=\frac{1}{\|\mathcal{P}_{\Omega}(D)\|_{F}}$, $\alpha_{max}=10^{10}$, and $\rho=1.1$.\\
\WHILE {not converged}
\STATE {Update $U_{k+1}$ by (15)}.
\STATE {Update $V_{k+1}$ by (17)}.
\STATE {Update $S_{k+1}$ by (22) and (24)}.
\STATE {Update the multiplier $Y_{k+1}$ by $Y_{k+1}=Y_{k}+\alpha_{k}(D-U_{k+1}V^{T}_{k+1}-S_{k+1})$.}
\STATE {Update $\alpha_{k+1}$ by $\alpha_{k+1}=\textup{min}(\rho\alpha_{k},\,\alpha_{max})$.}
\STATE {Check the convergence condition, $\|D-U_{k+1}V^{T}_{k+1}-S_{k+1}\|_{F}<\varepsilon$.}
\ENDWHILE
\end{algorithmic}
\end{algorithm}

\subsection{Extension for CPCP}
Algorithm 1 can be extended to solve the CPCP problem (8) with the complex operator $\mathcal{P}_{Q}$, as outlined in \textbf{Algorithm 2}, which is to optimize the following augmented Lagrange function
\begin{equation} \label{e4.25}
\begin{split}
\mathcal{F}_{\alpha}(U,V,S,Y)=&\lambda\|V\|_{*}+\|S\|_{1}+\langle Y,y-\mathcal{P}_{Q}(S+UV^{T})\rangle\\
&+\frac{\alpha}{2}\|y-\mathcal{P}_{Q}(S+UV^{T})\|^{2}_{2}.
\end{split}
\end{equation}
We minimize $\mathcal{F}_{\alpha}$ with respect to $(U, V, S)$ using a recently proposed linearization technique \cite{yang:lalad}, which resolves such problems where $\mathcal{P}_{Q}$ is not the identity operator. Specifically, for updating $U$ and $V$, let $T=UV^{T}$ and $g(T)=\frac{\alpha_{k}}{2}\|y-\mathcal{P}_{Q}(S_{k}+T)+Y_{k}/\alpha_{k}\|^{2}_{2}$, then $g(T)$ is approximated by
\begin{equation} \label{e4.26}
g(T)\approx g(T_{k})+\langle \nabla g(T_{k}),\,T-T_{k}\rangle+\tau\|T-T_{k}\|^{2}_{F},
\end{equation}
where $\nabla g(T_{k})=\alpha_{k}\mathcal{P}^{\star}_{Q}(\mathcal{P}_{Q}(T_{k}+S_{k})-y-Y_{k}/\alpha_{k})$, $\mathcal{P}^{\star}_{Q}$ is the adjoint operator of $\mathcal{P}_{Q}$, and $\tau$ is chosen as $\tau=1/\|\mathcal{P}^{\star}_{Q}\mathcal{P}_{Q}\|_{2}$ as in \cite{yang:lalad}, and $\|\cdot\|_{2}$ the spectral norm of a matrix, i.e., the largest singular value of a matrix.

Similarly, for updating $S$, let $T_{k+1}=U_{k+1}V^{T}_{k+1}$ and $h(S)=\frac{\alpha_{k}}{2}\|y-\mathcal{P}_{Q}(S+T_{k+1})+Y_{k}/\alpha_{k}\|^{2}_{2}$, then $h(S)$ is approximated by
\begin{equation} \label{e4.27}
h(S)\approx h(S_{k})+\langle \nabla h(S_{k}),\,S-S_{k}\rangle+\tau\|S-S_{k}\|^{2}_{F},
\end{equation}
where $\nabla h(S_{k})=\alpha_{k}\mathcal{P}^{\star}_{Q}(\mathcal{P}_{Q}(S_{k}+T_{k+1})-y-Y_{k}/\alpha_{k})$.

\begin{algorithm}[t]
\caption{Solving CPCP problem (8) via ADMM.}
\label{alg:Framwork2}
\renewcommand{\algorithmicrequire}{\textbf{Input:}}
\renewcommand{\algorithmicensure}{\textbf{Initialize:}}
\renewcommand{\algorithmicoutput}{\textbf{Output:}}
\begin{algorithmic}[1]
\REQUIRE $y\in\mathbb{R}^{p}$, $\mathcal{P}_{Q}$, and parameters $\lambda$ and $\varepsilon$.
\OUTPUT $U$, $V$ and $S$.\\
\ENSURE $U_{0}=\textup{eye}(m, d)$, $V_{0}=$\textbf{0}, $Y_{0}=$\textbf{0}, $\alpha_{0}=\frac{1}{\|y\|_{2}}$, $\alpha_{max}=10^{10}$, and $\rho=1.1$.\\
\WHILE {not converged}
\STATE {Update $U_{k+1}$ by $U_{k+1}=Q,\;\textup{QR}((U_{k}V^{T}_{k}-\nabla g(U_{k}V^{T}_{k})/\tau)V_{k})=QR$.}
\STATE {Update $V_{k+1}$ by $V^{T}_{k+1}=\textup{SVT}_{\lambda/\alpha_{k}}(U^{T}_{k+1}(U_{k}V^{T}_{k}-\nabla g(U_{k}V^{T}_{k})/\tau))$.}
\STATE {Update $S_{k+1}$ by $S_{k+1}=\mathcal{S}_{1/\alpha_{k}}(S_{k}-\nabla h(S_{k})/\tau)$.}
\STATE {Update the multiplier $Y_{k+1}$ by $Y_{k+1}=Y_{k}+\alpha_{k}(y-\mathcal{P}_{Q}(U_{k+1}V^{T}_{k+1}+S_{k+1}))$.}
\STATE {Update the parameter $\alpha_{k+1}$ by $\alpha_{k+1}=\textup{min}(\rho\alpha_{k},\,\alpha_{max})$.}
\STATE {Check the convergence condition,\\
$(\|T_{k+1}-T_{k}\|^{2}_{F}+\|S_{k+1}-S_{k}\|^{2}_{F})/(\|T_{k}\|^{2}_{F}+\|S_{k}\|^{2}_{F})<\varepsilon$.}
\ENDWHILE
\end{algorithmic}
\end{algorithm}

\subsection{Stopping Criteria and Rank Adjusting Strategy}
As the stopping criteria for terminating our RBF algorithms, we employ some gap criteria; that is, we stop Algorithm 1 when the current gap is satisfied as a given tolerance $\varepsilon$, i.e., $\|D-U_{k}V^{T}_{k}-S_{k}\|_{F}<\varepsilon$, and run Algorithm 2 when $(\|U_{k}V^{T}_{k}-U_{k-1}V^{T}_{k-1}\|^{2}_{F}+\|S_{k}-S_{k-1}\|^{2}_{F})/(\|U_{k-1}V^{T}_{k-1}\|^{2}_{F}+\|S_{k-1}\|^{2}_{F})<\varepsilon$.

In Algorithms 1 and 2, $d$ is one of the most important parameters. If $d$ is too small, it can cause underfitting and a large estimation error; but if $d$ is too large, it can cause overfitting and large deviation to the underlying low-rank matrix $L$. Fortunately, several works \cite{keshavan:mc, wen:nsor} have provided some matrix rank estimation strategies to compute a good value $r$ for the rank of the involved matrices. Thus, we only set a relatively large integer $d$ such that $d\geq r$. In addition, we provide a scheme to dynamically adjust the rank parameter $d$. Our scheme starts from an overestimated input, i.e., $d=\lfloor1.2r\rfloor$. Following \cite{keshavan:mc} and \cite{shen:alad}, we decease it aggressively once a dramatic change in the estimated rank of the product $U_{k}V^{T}_{k}$ is detected based on the eigenvalue decomposition which usually occurs after a few iterations. Specifically, we calculate the eigenvalues of $(U_{k}V^{T}_{k})^{T}U_{k}V^{T}_{k}=V_{k}U^{T}_{k}U_{k}V^{T}_{k}=V_{k}V^{T}_{k}$, which are assumed to be ordered as $\lambda_{1} \geq\lambda_{2} \geq\ldots\geq\lambda_{d}$. Since the product $V_{k}V^{T}_{k}$ and $V^{T}_{k}V_{k}$ have the same nonzero eigenvalues, it is much more efficient to compute the eigenvalues of the product $V^{T}_{k}V_{k}$. Then we compute the quotient sequence $\bar {\lambda}_{i}=\lambda_{i}/\lambda_{i+1},\,i=1,\ldots,d-1$. Suppose $\hat{r}=\mathop{\arg\max}_{1\leq i\leq d-1}\bar{\lambda}_{i}$. If the condition
\begin{displaymath}
\textup{gap}=\frac{(d-1)\bar{\lambda}_{\hat{r}}}{\sum_{i\neq\hat{r}}\bar{\lambda}_{i}}\geq 10,
\end{displaymath}
is satisfied, which means a ``big" jump between $\lambda_{\hat{r}}$ and $\lambda_{\hat{r}+1}$, then we reduce $d$ to $\hat{r}$, and this adjustment is performed only once.

\section{Theoretical Analysis and Applications}
In this section, we will present several theoretical properties of Algorithm 1. First, we provide the equivalent relationship analysis for our iterative solving scheme and the one in \cite{liu:as}, as shown by the following theorem.

\begin{theorem}
Let $(U^{*}_{k},V^{*}_{k},S^{*}_{k})$ be the solution of the subproblems (11), (12) and (13) at the k-th iteration, respectively, $Y^{*}_{k}=Y^{*}_{k-1}+\alpha_{k-1}(D-U^{*}_{k}(V^{*}_{k})^{T}-S^{*}_{k})$, and  $(U_{k},V_{k},S_{k},Y_{k})$ be generated by Algorithm 1 at the k-th iteration, $(k=1,\ldots,T)$. Then
\begin{enumerate}
  \item {$\exists O_{k}\in\mathcal{O}=\{M\in\mathbb{R}^{d\times d}|M^{T}M=I\}$ such that $U^{*}_{k}=U_{k}O_{k}$ and $V^{*}_{k}=V_{k}O_{k}$}.
  \item {$U^{*}_{k}(V^{*}_{k})^{T}=U_{k}V^{T}_{k}$, $\|V^{*}_{k}\|_{*}=\|V_{k}\|_{*}$}, $S^{*}_{k}=S_{k}$, and $Y^{*}_{k}=Y_{k}$.
\end{enumerate}
\end{theorem}

\textbf{Remark:} The proof of this theorem can be found in APPENDIX D. Since the Lagrange function (10) is determined by the product $UV^{T}$, $V$, $S$ and $Y$, the different values of $U$ and $V$ are essentially equivalent as long as the same product $UV^{T}$ and $\|V\|_{*}=\|V^{*}\|_{*}$. Meanwhile, our scheme replaces SVD by the QR decomposition, and can avoid the SVD computation for solving the optimization problem with the orthogonal constraint.

\subsection{Convergence Analysis}
The global convergence of our derived algorithm is guaranteed, as shown in the following lemmas and theorems.
\begin{lemma}
Let $(U_{k},V_{k},S_{k})$ be a sequence generated by Algorithm 1, then we have the following conclusions:
\begin{enumerate}
\item {$(U_{k},V_{k},S_{k})$ approaches to a feasible solution, i.e., $lim_{k\rightarrow \infty}\|D-U_{k}V^{T}_{k}-S_{k}\|_{F}=0$}.
\item {Both sequences $U_{k}V^{T}_{k}$ and ${S_{k}}$ are Cauchy sequences}.
\end{enumerate}
\end{lemma}

The detailed proofs of this lemma, the following lemma and theorems can be found in APPENDIX E. Lemma 8 ensures only that the feasibility of each solution has been assessed. In this paper, we want to show that it is possible to prove the local optimality of the solution produced by Algorithm 1. Let $k^{*}$ be the number of iterations needed by Algorithm 1 to stop, and $(U^{*},V^{*},S^{*})$ be defined by
\begin{displaymath}
U^{*}=U_{k^{*}+1},\;V^{*}=V_{k^{*}+1},\;S^{*}=S_{k^{*}+1}.
\end{displaymath}
In addition, let $Y^{*}$ (\textrm{resp.} $\widehat{Y}^{*}$) denote the Lagrange multiplier $Y_{k^{*}+1}$ (\textrm{resp.} $\widehat{Y}_{k^{*}+1}$) associated with $(U^{*},\,V^{*},\,S^{*})$, i.e., $Y^{*}=Y_{k^{*}+1},\;\widehat{Y}^{*}=\widehat{Y}_{k^{*}+1}$, where $\widehat{Y}_{k^{*}+1}=Y_{k^{*}}+\alpha_{k^{*}}(D-U_{k^{*}+1}V^{T}_{k^{*}+1}-S_{k^{*}})$.

\begin{lemma}
For the solution $(U^{*},V^{*},S^{*})$ generated by Algorithm 1, then we have the following conclusion:
\begin{displaymath}
\begin{split}
\|\mathcal{P}_{\Omega}(S)\|_{1}+\lambda\|V\|_{*}\geq\|\mathcal{P}_{\Omega}(S^{*})\|_{1}+\lambda\|V^{*}\|_{*}+\langle \widehat{Y}^{*}-Y^{*},\;UV^{T}-U^{*}(V^{*})^{T}\rangle-mn\varepsilon
\end{split}
\end{displaymath}
holds for any feasible solution $(U,\,V,\,S)$ to (9).
\end{lemma}

To reach the global optimality of (9) based on the above lemma, it is required to show that the term $\langle \widehat{Y}^{*}-Y^{*},\;UV^{T}-U^{*}(V^{*})^{T}\rangle$ diminishes. Since
\begin{displaymath}
\begin{split}
\|Y^{*}-\widehat{Y}^{*}\|_{2}\leq \sqrt{mn}\|Y^{*}-\widehat{Y}^{*}\|_{\infty},
\end{split}
\end{displaymath}
and by Lemma 13 (Please see APPENDIX E), we have
\begin{displaymath}
\|Y^{*}-\widehat{Y}^{*}\|_{\infty}=\|\mathcal{P}_{\Omega}(Y^{*})-\widehat{Y}^{*}\|_{\infty}\leq\|\mathcal{P}_{\Omega}(Y^{*})\|_{\infty}+\|\widehat{Y}^{*}\|_{\infty}\leq 1+\lambda,
\end{displaymath}
which means that $\|Y^{*}-\widehat{Y}^{*}\|_{\infty}$ is bounded. By setting the parameter $\rho$ to be relatively small as in \cite{liu:as}, $\|Y^{*}-\widehat{Y}^{*}\|_{\infty}$ is small, which means that  $\|Y^{*}-\widehat{Y}^{*}\|_{2}$ is also small. Let $\varepsilon_{1}=\|Y^{*}-\widehat{Y}^{*}\|_{2}$, then we have the following theorems.

\begin{theorem}
Let $f^{g}$ be the globally optimal objective function value of (9), and $f^{*}=\|\mathcal{P}_{\Omega}(S^{*})\|_{1}+\lambda\|V^{*}\|_{*}$ be the objective function value generated by Algorithm 1. We have that
\begin{displaymath}
f^{*}\leq f^{g}+c_{1}\varepsilon_{1}+mn\varepsilon,
\end{displaymath}
where $c_{1}$ is a constant defined by
\begin{displaymath}
c_{1}=\frac{mn}{\lambda}\|\mathcal{P}_{\Omega}(D)\|_{F}(\frac{\rho(1+\rho)}{\rho-1}+\frac{1}{2\rho^{k^{*}}})+\frac{\|\mathcal{P}_{\Omega}(D)\|_{1}}{\lambda}.
\end{displaymath}
\end{theorem}

\begin{theorem}
Suppose $(L^{0},\,S^{0})$ is an optimal solution to the RMC problem (5), $\textrm{rank}(L^{0})=r$ and $f^{0}=\|\mathcal{P}_{\Omega}(S^{0})\|_{1}+\lambda\|L^{0}\|_{*}$. Let $f^{*}=\|\mathcal{P}_{\Omega}(S^{*})\|_{1}+\lambda\|U^{*}V^{*}\|_{*}$ be the objective function value generated by Algorithm 1 with parameter $d>0$. We have that
\begin{displaymath}
f^{0}\leq f^{*}\leq f^{0}+c_{1}\varepsilon_{1}+mn\varepsilon+(\sqrt{mn}-\lambda)\sigma_{d+1}\max(r-d,\,0),
\end{displaymath}
where $\sigma_{1}\geq\sigma_{2}\geq\ldots$ are the singular values of $L^{0}$.
\end{theorem}

Since the rank parameter $d$ is set to be higher than the rank of the optimal solution to the RMC problem (5), i.e., $d\geq r$, Theorem 11 directly concludes that
\begin{displaymath}
f^{0}\leq f^{*}\leq f^{0}+c_{1}\varepsilon_{1}+mn\varepsilon.
\end{displaymath}
Moreover, the value of $\varepsilon$ can be set to be arbitrarily small, and the second term involving $\varepsilon_{1}$ diminishes. Hence, for the solution $(U^{*},V^{*},S^{*})$ generated by Algorithm 1, a solution to the RMC problem (5) can be achieved by computing $L^{*}=U^{*}(V^{*})^{T}$.

\subsection{Complexity Analysis}
We also discuss the time complexity of our RBF algorithm. For the RMC problem (7), the main running time of our RBF algorithm is consumed by performing SVD on the small matrix of size $n\times d$, the QR decomposition of the matrix $P_{k}V_{k}$, and some matrix multiplications. In (17), the time complexity of performing SVD is $O(d^{2}n)$. The time complexity of QR decomposition and matrix multiplications is $O(d^{2}m+mnd)$. The total time complexity of our RBF algorithm for solving both problems (3) and (7) is $O(t(d^{2}n+d^{2}m+mnd))$ (usually $d\ll n\leq m$), where $t$ is the number of iterations.

\subsection{Applications of Matrix Completion}
As our RBF framework introduced for robust matrix factorization is general, there are many possible extensions of our methodology. In this section, we outline a novel result and methodology for one extension we consider most important: low-rank matrix completion. The space limit refrains us from fully describing each development, but we try to give readers enough details to understand and use each of these applications.

By introducing an auxiliary variable $L$, the low-rank matrix completion problem can be written into the following form,
\begin{equation}
\begin{split}
&\min_{U,\,V,\,L}\,\frac{1}{2}\|\mathcal{P}_{\Omega}(D)-\mathcal{P}_{\Omega}(L)\|^{2}_{F}+\lambda\|V\|_{*},\\
&\;\textup{s.t.},\,L=UV^{T},\,U^{T}U=I.
\end{split}
\end{equation}
Similar to Algorithm 1, we can present an efficient ADMM scheme to solve the matrix completion problem (28). This algorithm can also be easily used to solve the low-rank matrix factorization problem, where all entries of the given matrix are observed.

\section{Experimental Evaluation}
We now evaluate the effectiveness and efficiency of our RBF method for RMC and CPCP problems such as text removal, background modeling, face reconstruction, and collaborative filtering. We ran experiments on an Intel(R) Core (TM) i5-4570 (3.20 GHz) PC running Windows 7 with 8GB main memory.

\subsection{Text Removal}
We first conduct an experiment by considering a simulated task on artificially generated data, whose goal is to remove some generated text from an image. The ground-truth image is of size $256\times222$ with rank equal to 10 for the data matrix. we then add to the image a short phase in text form which plays the role of outliers. Fig.\ 2 shows the image together with the clean image and outliers mask. For fairness, we set the rank of all the algorithms to 20, which is two times the true rank of the underlying matrix. The input data are generated by setting 30\% of the randomly selected pixels of the image as missing entries. We compare our RBF method with the state-of-the-art methods, including PCP \cite{candes:rpca}, SpaRCS\footnote{\url{http://www.ece.rice.edu/~aew2/sparcs.html}} \cite{waters:rcs}, RegL1\footnote{\url{https://sites.google.com/site/yinqiangzheng/}} \cite{zheng:rma} and BF-ALM \cite{cabral:nnbf}. We set the regularization parameter $\lambda=\sqrt{\max(m,n)}$ for RegL1 and RBF, and the stopping tolerance $\varepsilon=10^{-4}$ for all algorithms in this section.

\begin{figure}[t]
\centering
\subfigure{\includegraphics[width=0.215\linewidth]{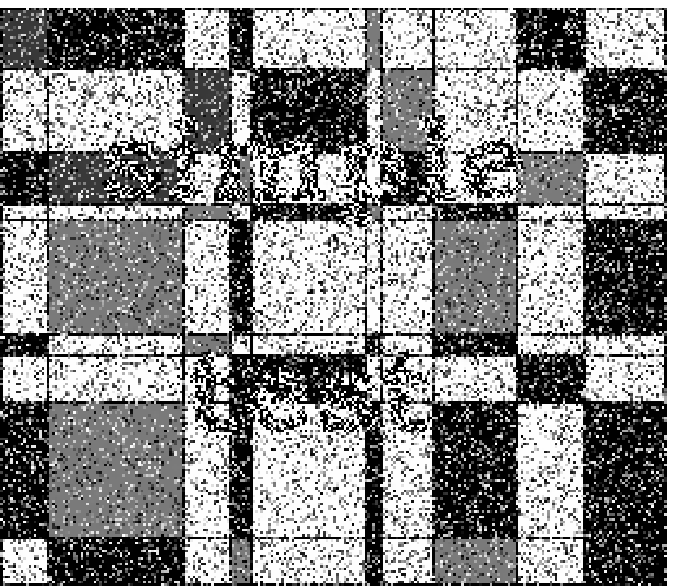}}\;\;
\subfigure{\includegraphics[width=0.215\linewidth]{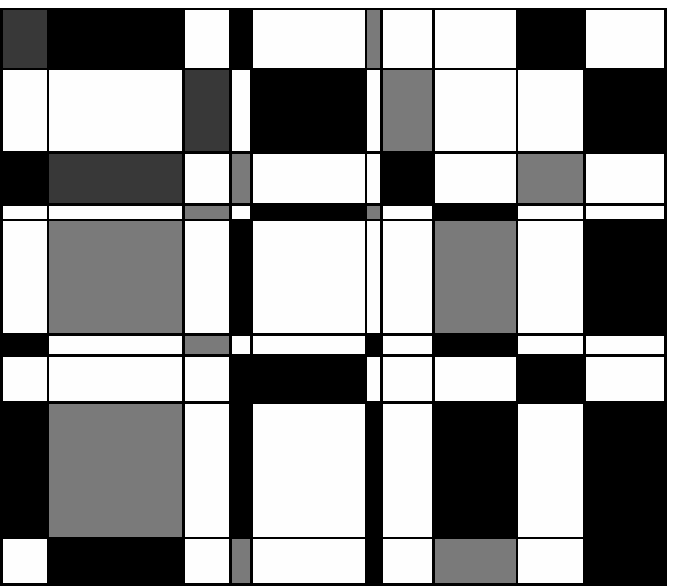}}\;\;
\subfigure{\includegraphics[width=0.215\linewidth]{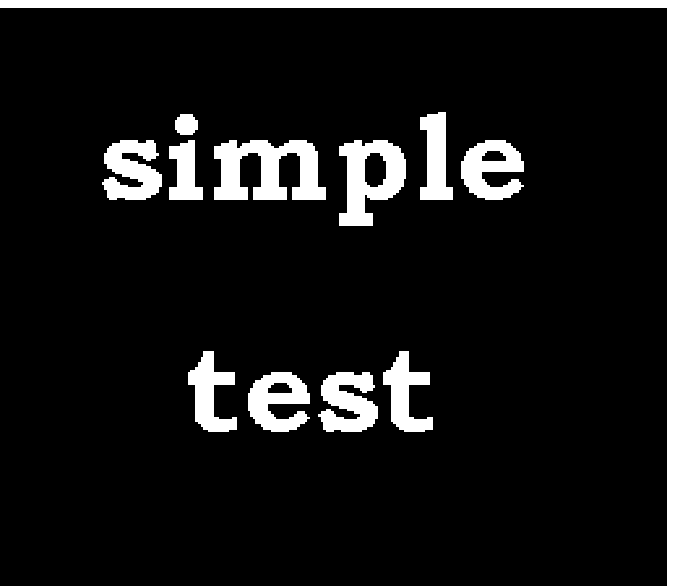}}\\
(a)\;\;\,\;\;\quad\quad\quad\quad\quad\quad(b)\;\;\,\;\;\quad\quad\quad\quad\quad\quad(c)
\caption{Image used in text removal experiment: (a) Input image; (b) Original image; (c) Outlier mask.}
\label{fig_sim}
\end{figure}

The results obtained by different methods are visually shown in Fig.\ 3, where the outlier detection accuracy (the score Area Under the receiver operating characteristic Curve, AUC) and the error of low-rank component recovery (i.e., $\|D-L\|_{F}/\|D\|_{F}$, where $D$ and $L$ denote the ground-truth image matrix and the recovered image matrix, respectively) are also presented. As far as low-rank matrix recovery is concerned, these RMC methods including SpaRCS, RegL1, BF-ALM and RBF, outperform PCP, not only visually but also quantitatively. For outlier detection, it can be seen that our RBF method significantly performs better than the other methods. In short, RBF significantly outperforms PCP, RegL1, BF-ALM and SpaRCS in terms of both low-rank matrix recovery and spare outlier identification. Moreover, the running time of PCP, SpaRCS, RegL1, BF-ALM and RBF is 36.25sec, 15.68sec, 26.85sec, 6.36sec and 0.87sec, respectively.

\begin{figure}[t]
\centering
\subfigure{\includegraphics[width=0.181\linewidth]{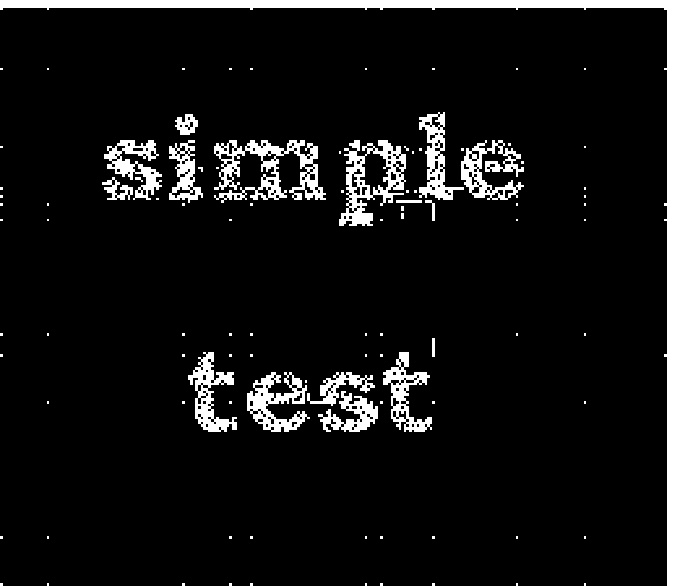}}
\subfigure{\includegraphics[width=0.181\linewidth]{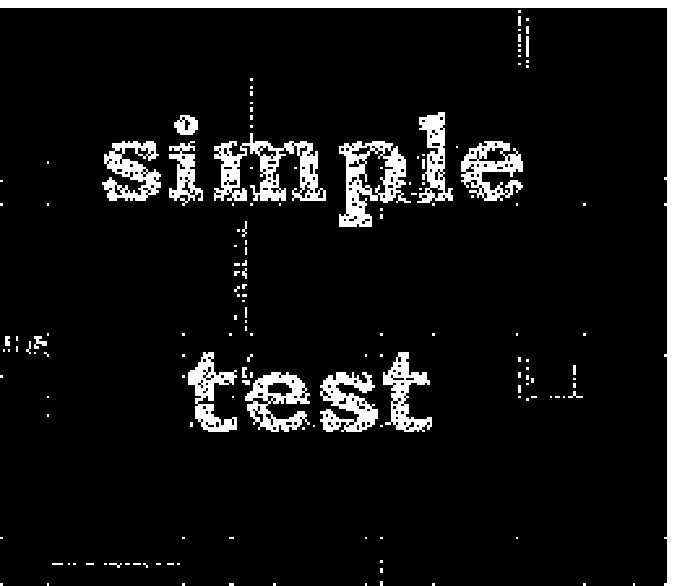}}
\subfigure{\includegraphics[width=0.181\linewidth]{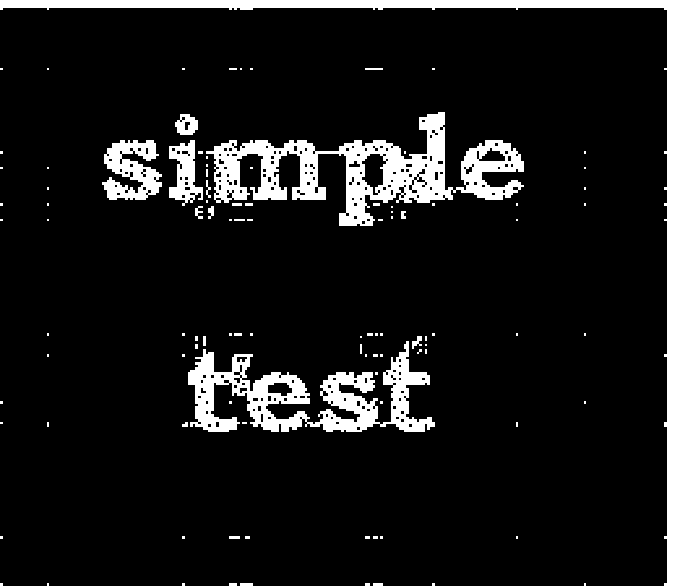}}
\subfigure{\includegraphics[width=0.181\linewidth]{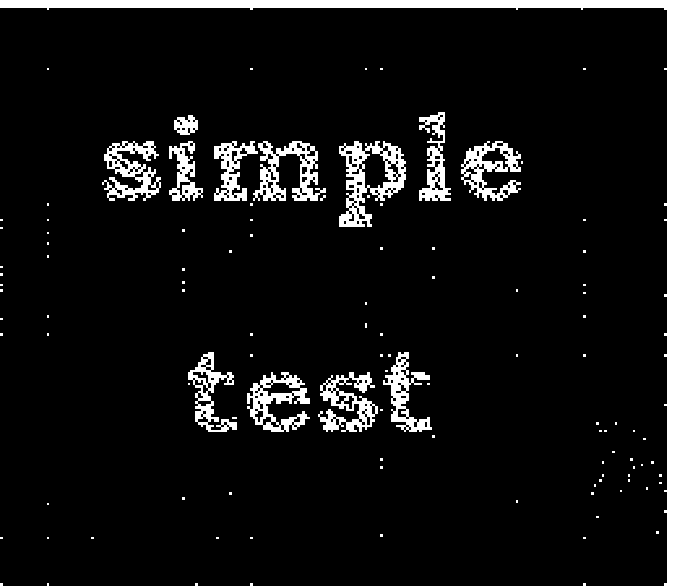}}
\subfigure{\includegraphics[width=0.181\linewidth]{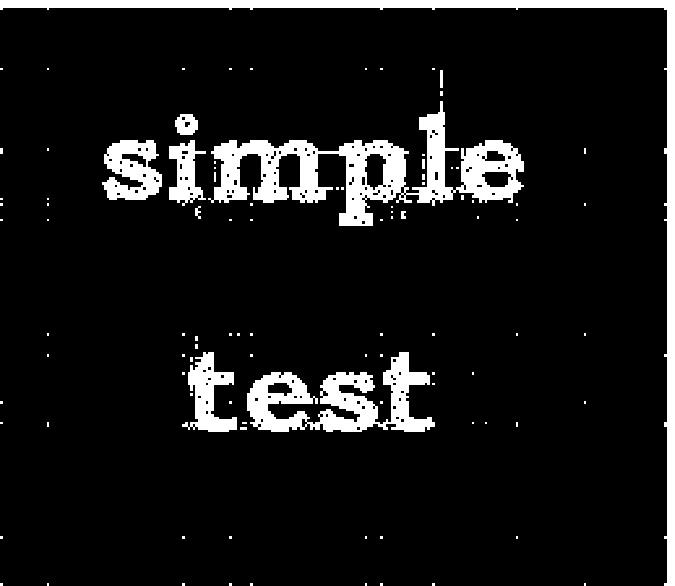}}
\subfigure{\includegraphics[width=0.181\linewidth]{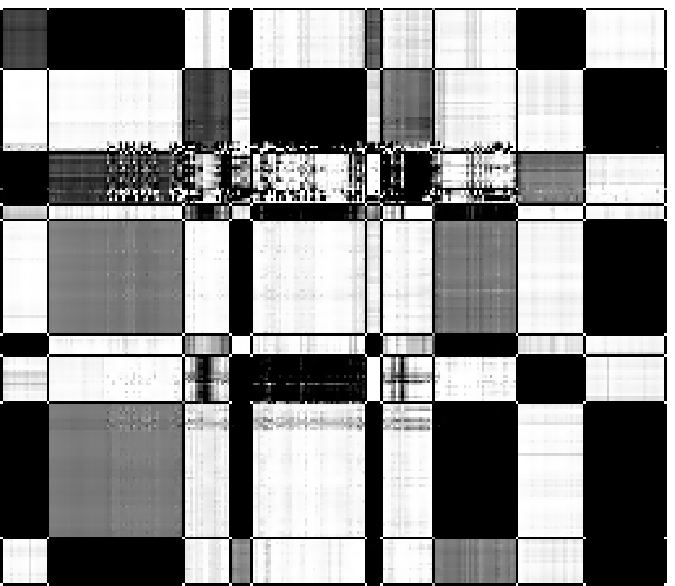}}
\subfigure{\includegraphics[width=0.181\linewidth]{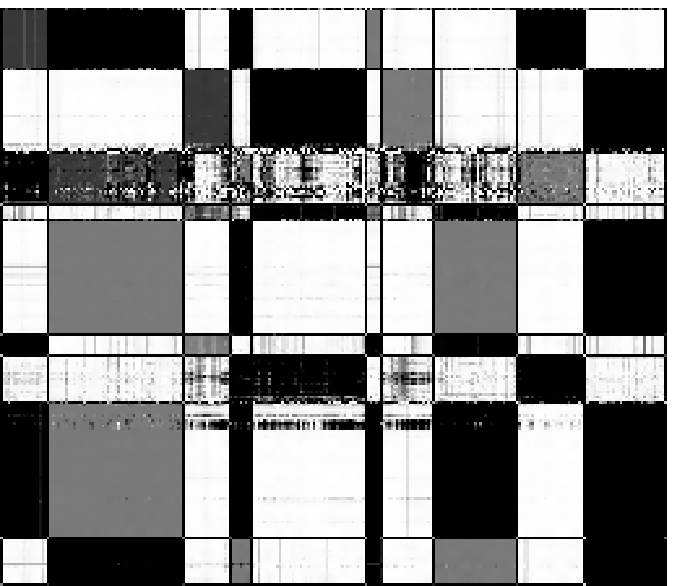}}
\subfigure{\includegraphics[width=0.181\linewidth]{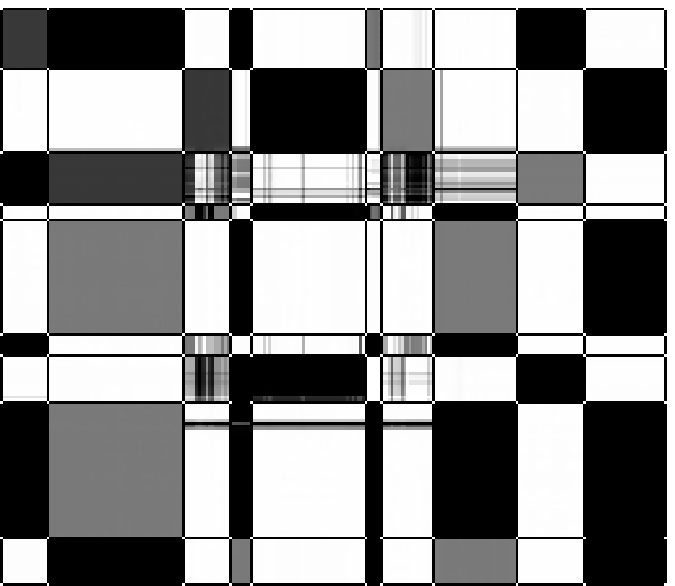}}
\subfigure{\includegraphics[width=0.181\linewidth]{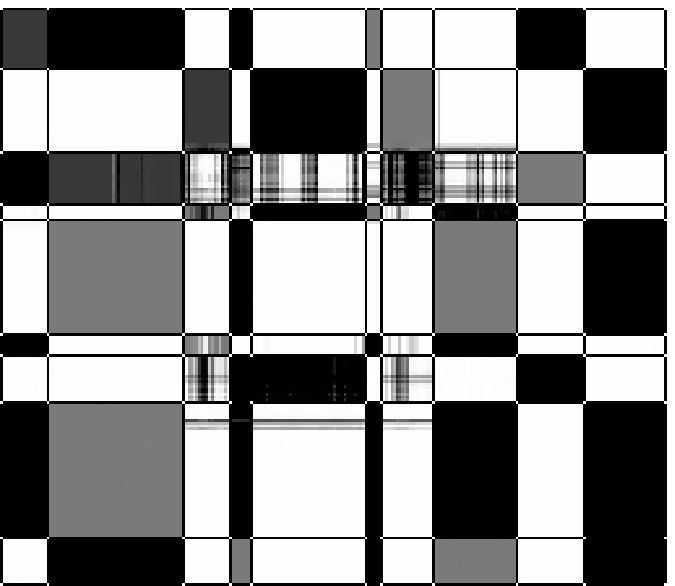}}
\subfigure{\includegraphics[width=0.181\linewidth]{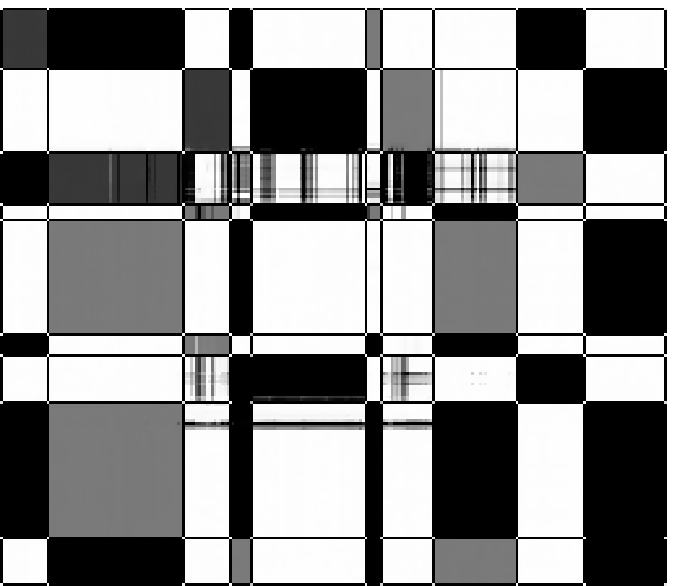}}\\
(a)\,\qquad\quad\quad\quad\quad(b)\,\qquad\quad\quad\quad\quad(c)\,\qquad\quad\quad\quad\quad(d)\,\qquad\quad\quad\quad\quad(e)
\caption{Text removal results. The first row shows the foreground masks and the second row shows the recovered background images: (a) PCP (AUC: 0.8558; Error: 0.2516); (b) SpaRCS (AUC: 0.8665; Error: 0.2416); (c) RegL1 (AUC: 0.8792; Error: 0.2291); (d) BF-ALM (AUC: 0.8568; Error: 0.2435); (e) RBF (AUC: 0.9227; Error: 0.1844).}
\label{fig_sim3}
\end{figure}

We further evaluate the robustness of our RBF method with respect to the regularization parameter $\lambda$ and the given rank variations. We conduct some experiments on the artificially generated data, and illustrate the outlier detection accuracy (AUC) and the error (Error) of low-rank component recovery of PCP, SpaRCS, RegL1 and our RBF method, where the given rank of SpaRCS, RegL1 and our RBF method is chosen from $\{20,25,\cdots,60\}$, and the regularization parameter $\lambda$ of PCP, RegL1 and RBF is chosen from the grid $\{1,2.5,5,7.5,10,25,50,75,100\}$. Notice that because BF-ALM and RegL1 achieve very similar results, we do not provide the results of the former in the following. The average AUC and Error results of 10 independent runs are shown in Figs.\ 4 and 5, from which we can see that our RBF method performs much more robust than SpaRCS and RegL1 with respect to the given rank. Moreover, our RBF method is much more robust than PCP and RegL1 against the regularization parameter $\lambda$.

\begin{figure}[t]
\centering
\includegraphics[width=0.46\linewidth]{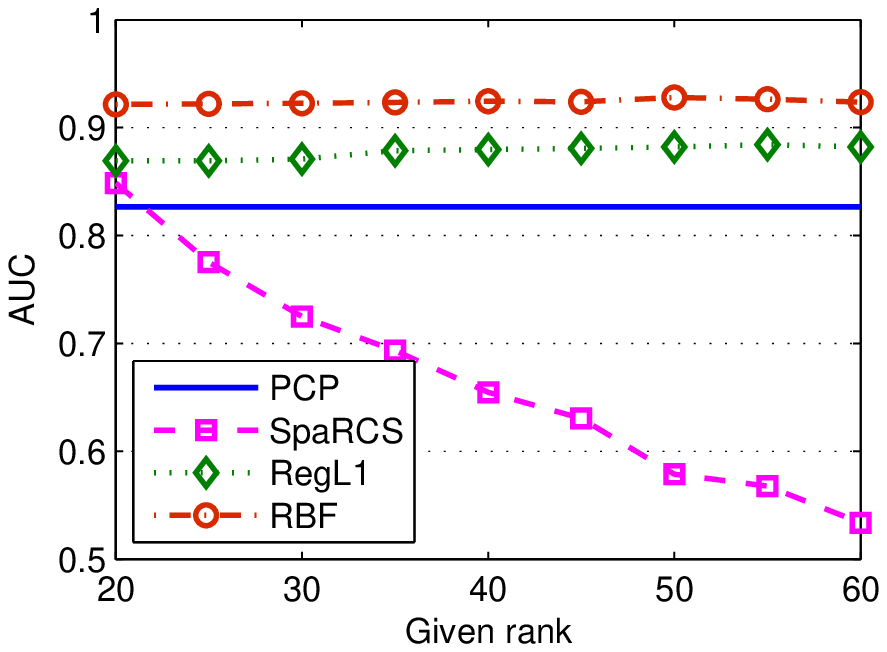}\label{fig_first_case}
\includegraphics[width=0.46\linewidth]{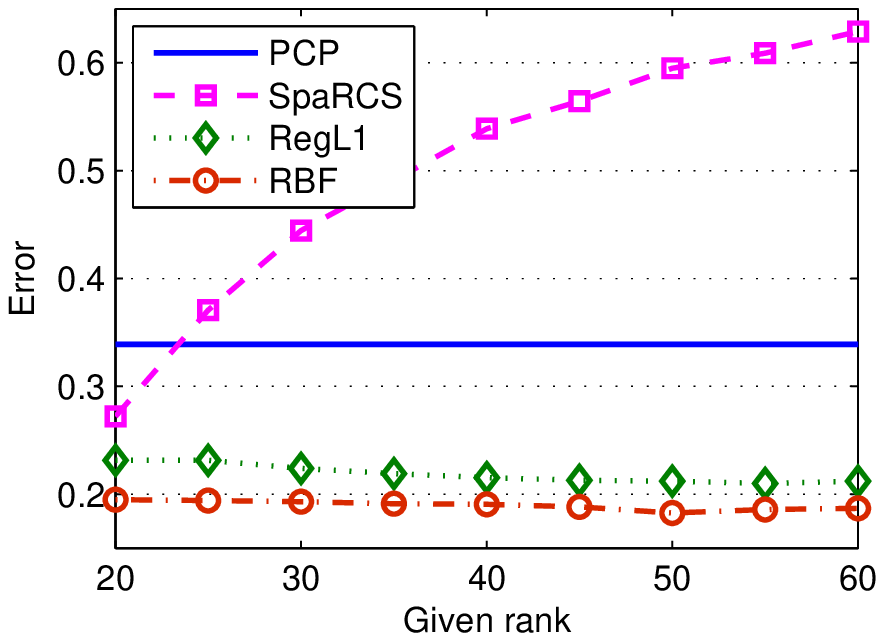}\label{fig_second_case}
\centering
\caption{Comparison of PCP, SpaRCS, RegL1 and our RBF method in terms of AUC (Left) and Error (Right) on the artificially generated data with varying ranks.}
\label{fig_sim}
\end{figure}

\begin{figure}[t]
\centering
\includegraphics[width=0.46\linewidth]{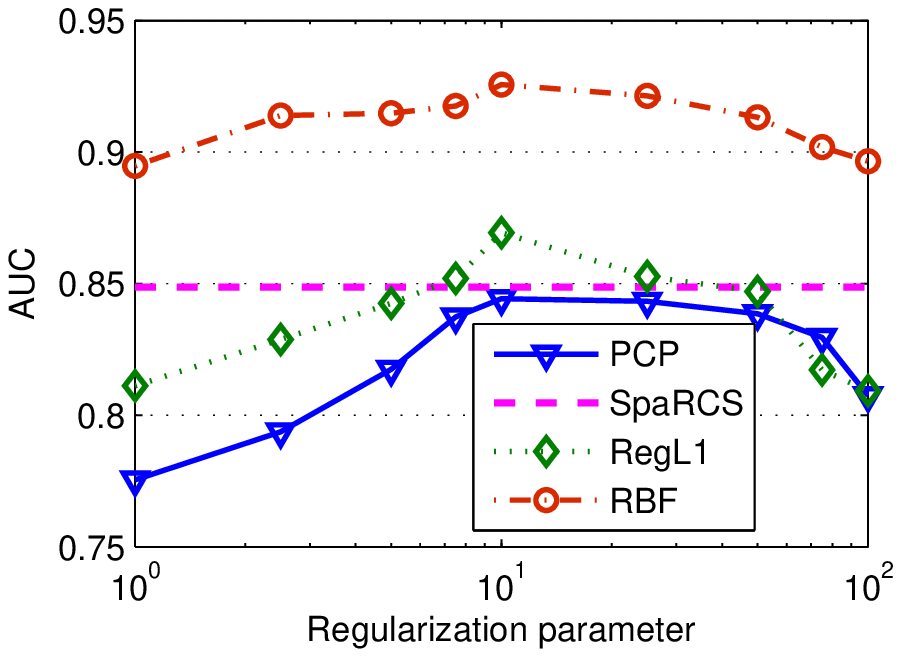}\label{fig_first_case}
\includegraphics[width=0.46\linewidth]{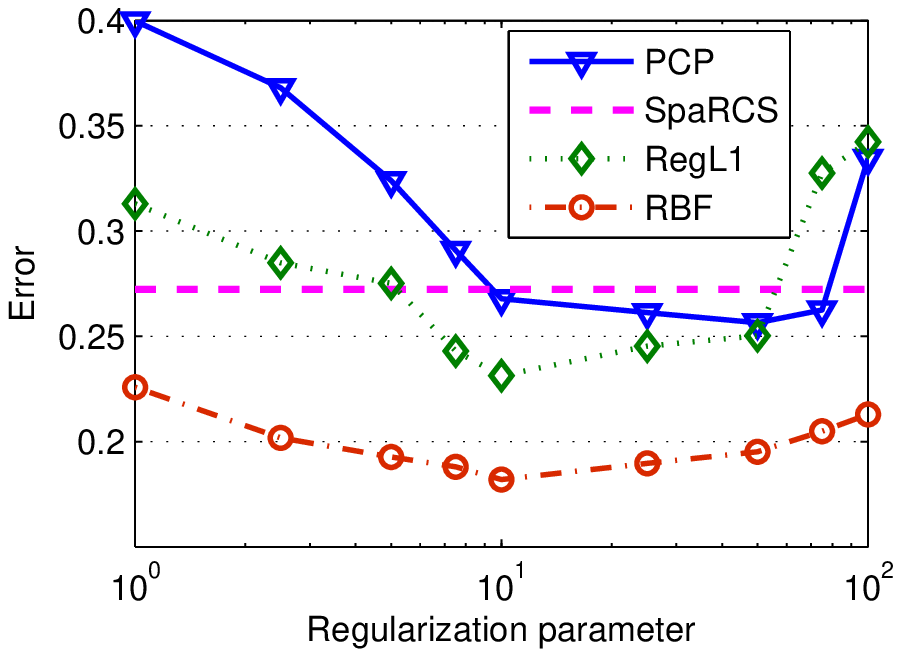}\label{fig_second_case}
\centering
\caption{Comparison of PCP, SpaRCS, RegL1 and our RBF method in terms of AUC (Left) and Error (Right) on the artificially generated data with varying regularization parameters.}
\label{fig_sim}
\end{figure}

\subsection{Background Modeling}
In this experiment we test our RBF method on real surveillance videos for object detection and background subtraction as a RPCA plus matrix completion problem. Background modeling is a crucial task for motion segmentation in surveillance videos. A video sequence satisfies the low-rank and sparse structures, because the background of all the frames is controlled by few factors and hence exhibits low-rank property, and the foreground is detected by identifying spatially localized sparse residuals \cite{wright:rpca, candes:rpca, wang:wvs}. We test our RBF method on four color surveillance videos: Bootstrap, Lobby, Hall and Mall databases\footnote{\url{http://perception.i2r.a-star.edu.sg/bkmodel/bkindex}}. The data matrix \emph{D} consists of the first 400 frames of size $144\times176$. Since all the original videos have colors, we first reshape every frame of the video into a long column vector and then collect all the columns into a data matrix \emph{D} with size of $76032\times400$. Moreover, the input data is generated by setting 10\% of the randomly selected pixels of each frame as missing entries.

Fig.\ 6 illustrates the background extraction results on the Bootstrap data set, where the first and fourth columns represent the input images with missing data, the second and fifth columns show the low-rank recoveries, and the third and sixth columns show the sparse components. It is clear that the background can be effectively extracted by our RBF method, RegL1 and GRASTA\footnote{\url{https://sites.google.com/site/hejunzz/grasta}} \cite{he:online}. Notice that SpaRCS could not yield experimental results on these databases because they ran out of memory. Moreover, we can see that the decomposition results of our RBF method,  especially the recovered low-rank components, are slightly better than that of RegL1 and GRASTA. We also report the running time in Table 1, from which we can see that RBF is more than 3 times faster than GRASTA and more than 2 times faster than RegL1. This further shows that our RBF method has very good scalability and can address large-scale problems.

\begin{table}[!ht]
\caption{Comparison of time costs in CPU seconds of GRASTA, RegL1 and RBF on background modeling data sets.}
\centering
\begin{tabular} {l|c|ccc}
\hline
\ Datasets           & Sizes        & GRASTA  & RegL1   & RBF\\
\hline
\ Bootstrap     &$57,600\times400$   &153.65   &93.17  &38.32\\
\ Lobby         &$61,440\times400$   &187.43   &139.83 &50.08\\
\ Hall	        &$76,032\times400$   &315.11   &153.45 &67.73\\
\ Mall	        &$245,760\times200$  &493.92   &\;----\; &94.59\\
\hline
\end{tabular}
\end{table}

\begin{figure}[!ht]
\centering
\subfigure{\includegraphics[width=0.158\linewidth]{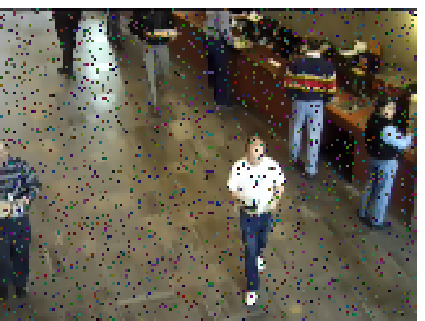}}
\subfigure{\includegraphics[width=0.158\linewidth]{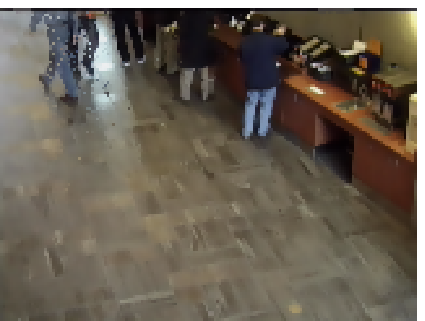}}
\subfigure{\includegraphics[width=0.158\linewidth]{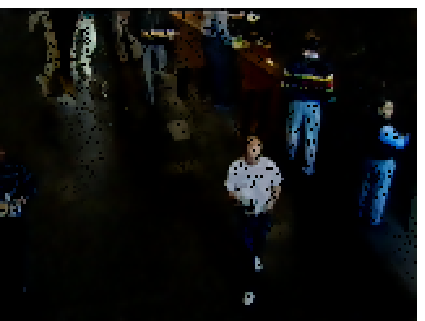}}\;\;
\subfigure{\includegraphics[width=0.158\linewidth]{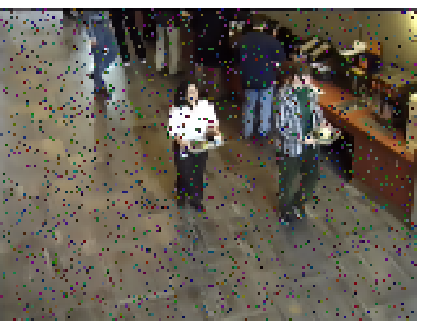}}
\subfigure{\includegraphics[width=0.158\linewidth]{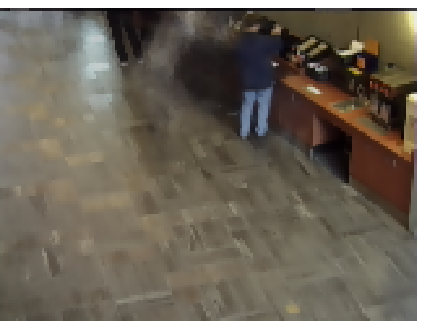}}
\subfigure{\includegraphics[width=0.158\linewidth]{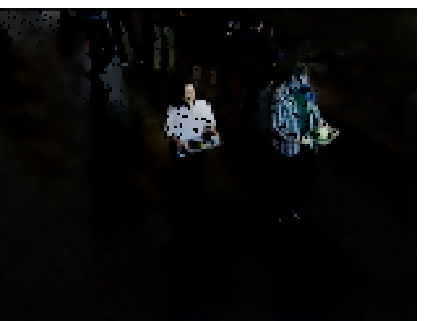}}
\subfigure{\includegraphics[width=0.158\linewidth]{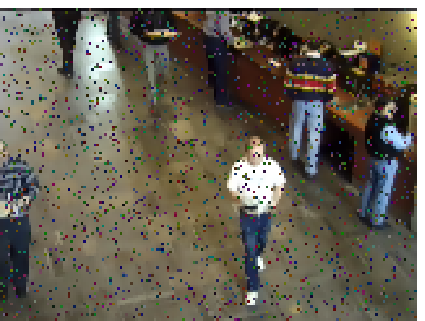}}
\subfigure{\includegraphics[width=0.158\linewidth]{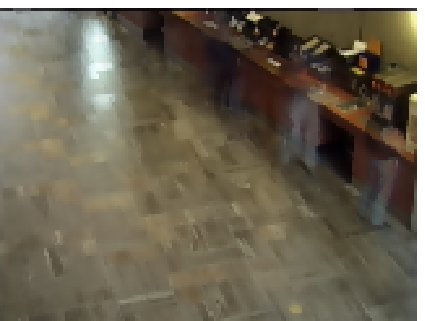}}
\subfigure{\includegraphics[width=0.158\linewidth]{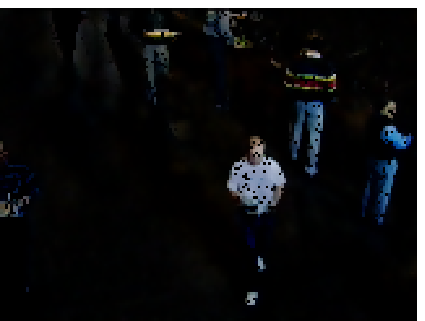}}\;\;
\subfigure{\includegraphics[width=0.158\linewidth]{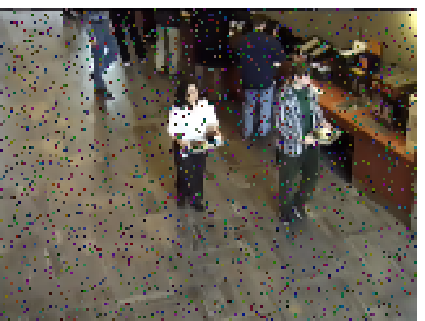}}
\subfigure{\includegraphics[width=0.158\linewidth]{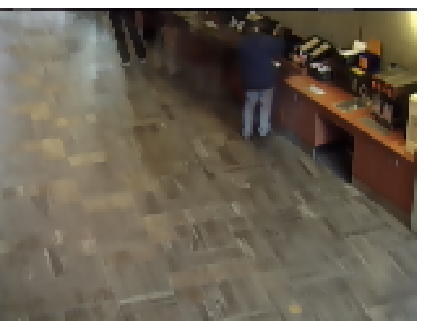}}
\subfigure{\includegraphics[width=0.158\linewidth]{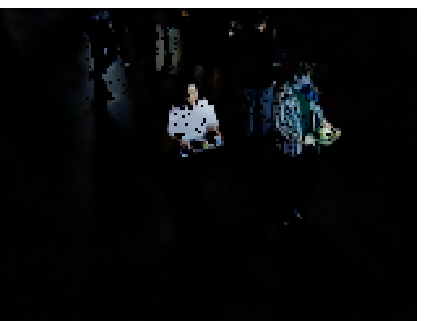}}
\subfigure{\includegraphics[width=0.158\linewidth]{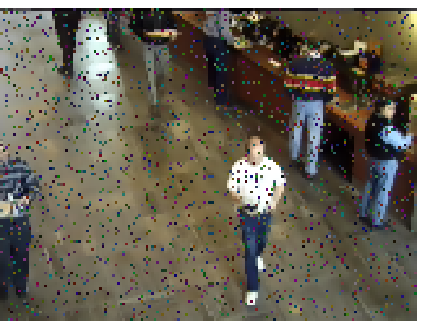}}
\subfigure{\includegraphics[width=0.158\linewidth]{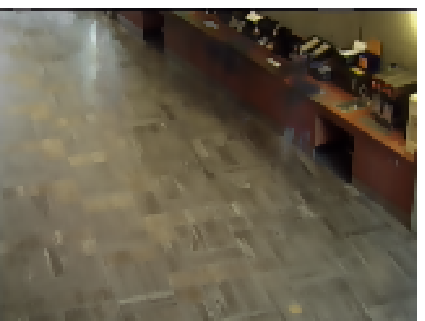}}
\subfigure{\includegraphics[width=0.158\linewidth]{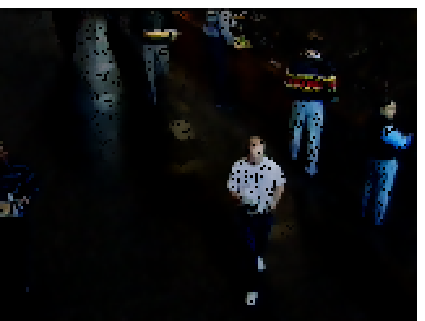}}\;\;
\subfigure{\includegraphics[width=0.158\linewidth]{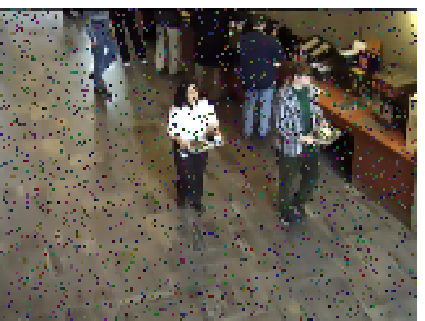}}
\subfigure{\includegraphics[width=0.158\linewidth]{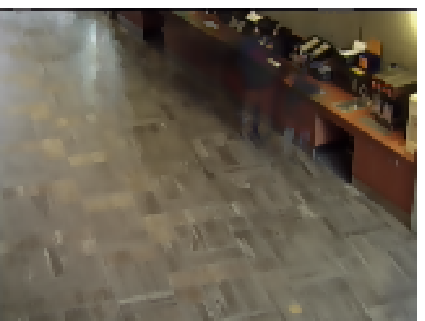}}
\subfigure{\includegraphics[width=0.158\linewidth]{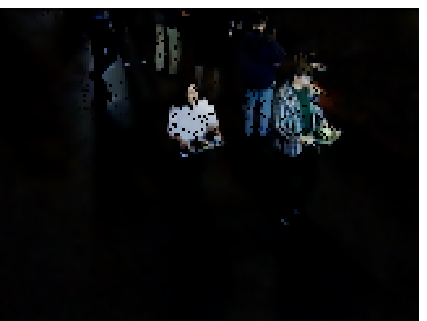}}
\caption{Background extraction results of different algorithms on the Bootstrap data set, where the first, second and last rows show the recovered low-rank and sparse images by GRASTA, RegL1 and RBF, respectively.}
\label{fig_6}
\end{figure}

\subsection{Face Reconstruction}
We also test our RBF method for the face reconstruction problems with the incomplete and corrupted face data or a small set of linear measurements $y$ as in \cite{wright:cpcp}, respectively. The face database used here is a part of Extended Yale Face Database B \cite{lee:fr} with the large corruptions. The face images can often be decomposed as a low-rank part, capturing the face appearances under different illuminations, and a sparse component, representing varying illumination conditions and heavily ``shadows". The resolution of all images is $192\times168$ and the pixel values are normalized to [0, 1], then the pixel values are used to form data vectors of dimension 32,256. The input data are generated by setting 40\% of the randomly selected pixels of each image as missing entries.

Fig.\ 7 shows some original and reconstructed images by RBF, PCP, RegL1 and CWM\footnote{\url{http://www4.comp.polyu.edu.hk/~cslzhang/papers.htm}} \cite{meng:cwm}, where the average computational time of all these algorithms on each people's faces is presented. It can be observed that RBF performs better than the other methods not only visually but also efficiently, and effectively eliminates the heavy noise and ``shadows" and simultaneously completes the missing entries. In other words, RBF can achieve the latent features underlying the original images regardless of the observed data corrupted by outliers or missing values.

\begin{figure}[t]
\centering
\subfigure{\includegraphics[width=0.16\linewidth]{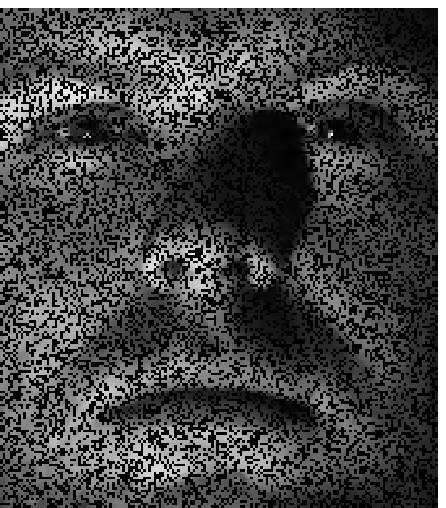}\label{fig_first_case}}
\subfigure{\includegraphics[width=0.16\linewidth]{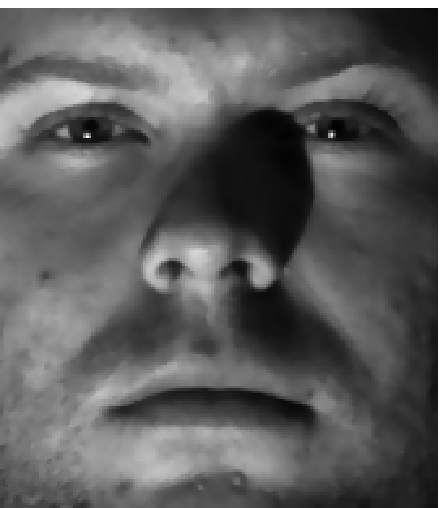}\label{fig_second_case}}
\subfigure{\includegraphics[width=0.16\linewidth]{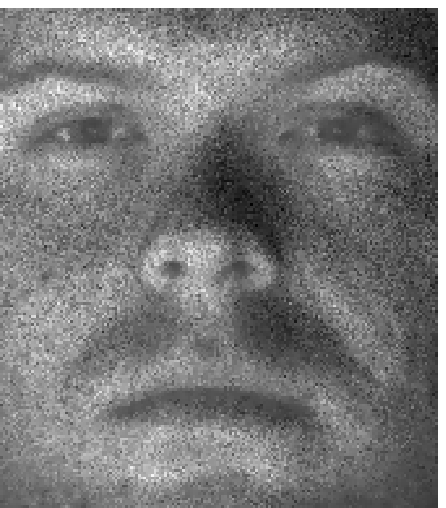}\label{fig_first_case}}
\subfigure{\includegraphics[width=0.16\linewidth]{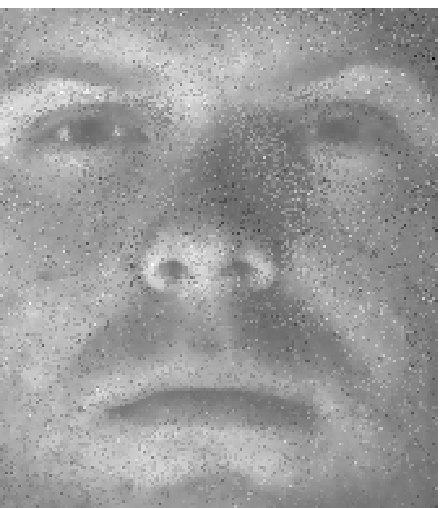}\label{fig_first_case}}
\subfigure{\includegraphics[width=0.16\linewidth]{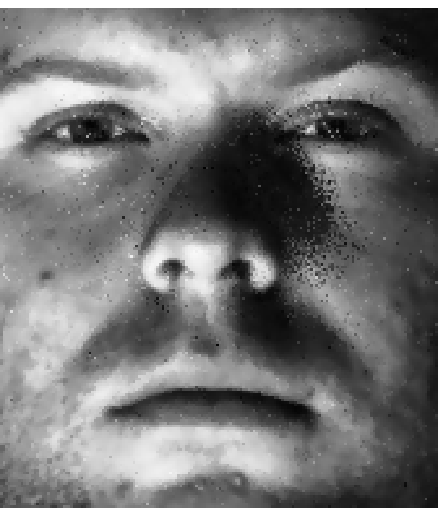}\label{fig_fourth_case}}
\subfigure{\includegraphics[width=0.16\linewidth]{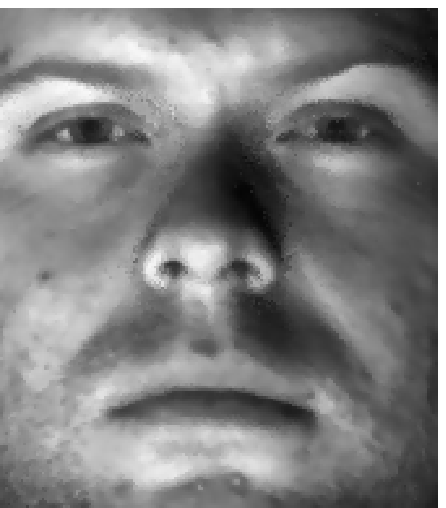}\label{fig_fourth_case}}
\subfigure{\includegraphics[width=0.16\linewidth]{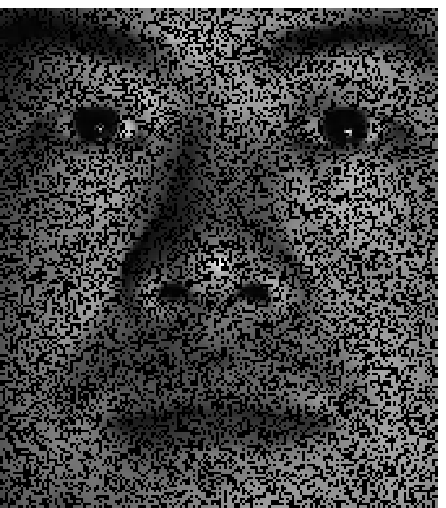}\label{fig_first_case}}
\subfigure{\includegraphics[width=0.16\linewidth]{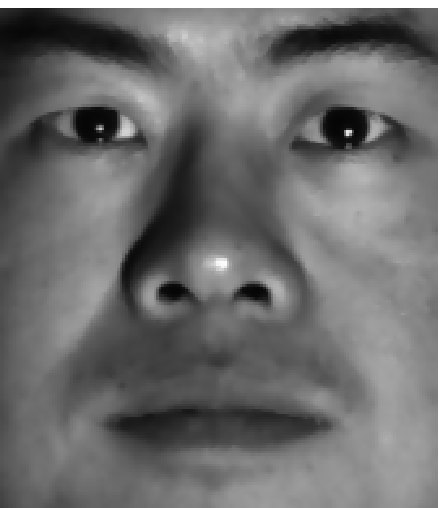}\label{fig_second_case}}
\subfigure{\includegraphics[width=0.16\linewidth]{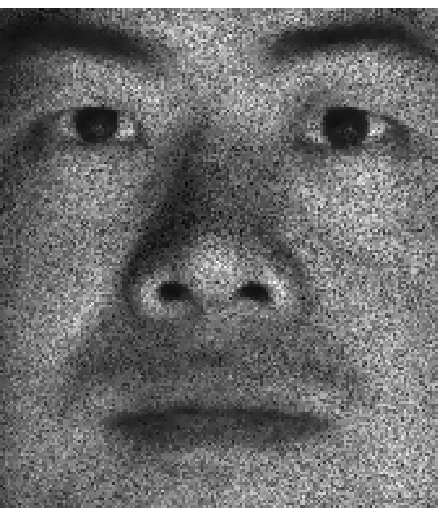}\label{fig_first_case}}
\subfigure{\includegraphics[width=0.16\linewidth]{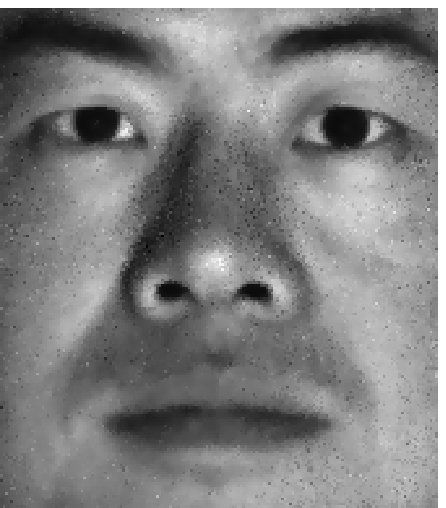}\label{fig_first_case}}
\subfigure{\includegraphics[width=0.16\linewidth]{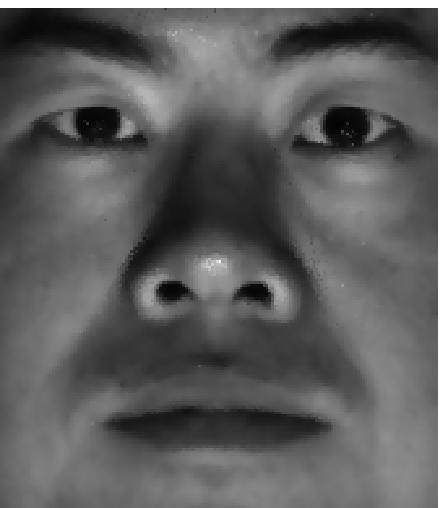}\label{fig_fourth_case}}
\subfigure{\includegraphics[width=0.16\linewidth]{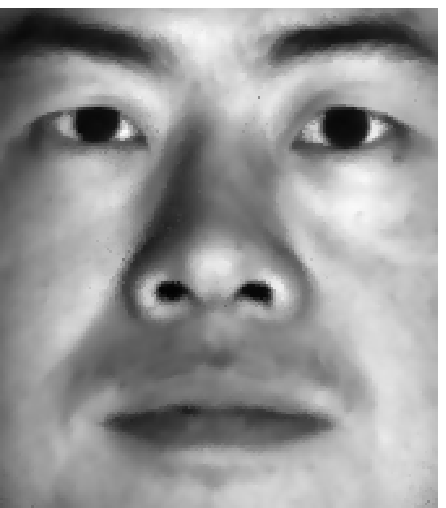}\label{fig_fourth_case}}
\subfigure{\includegraphics[width=0.16\linewidth]{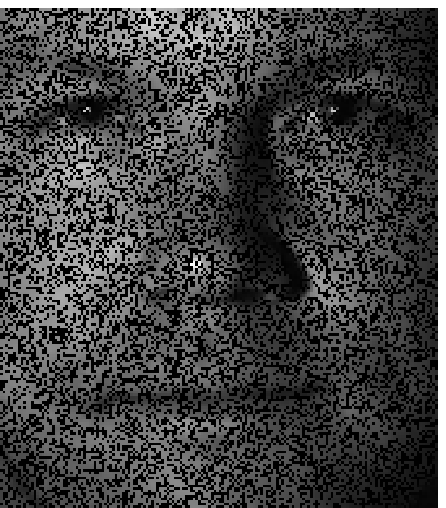}\label{fig_first_case}}
\subfigure{\includegraphics[width=0.16\linewidth]{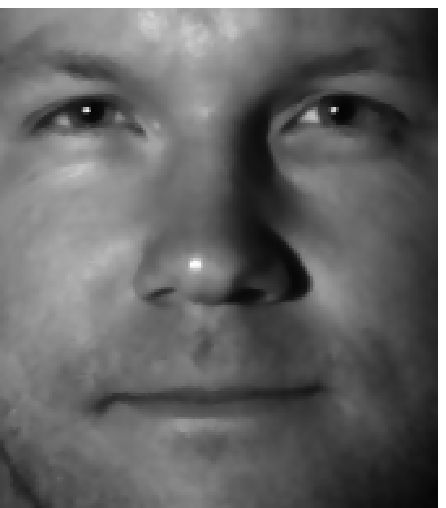}\label{fig_second_case}}
\subfigure{\includegraphics[width=0.16\linewidth]{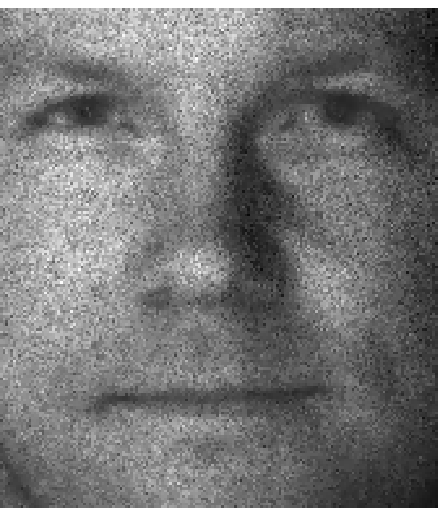}\label{fig_first_case}}
\subfigure{\includegraphics[width=0.16\linewidth]{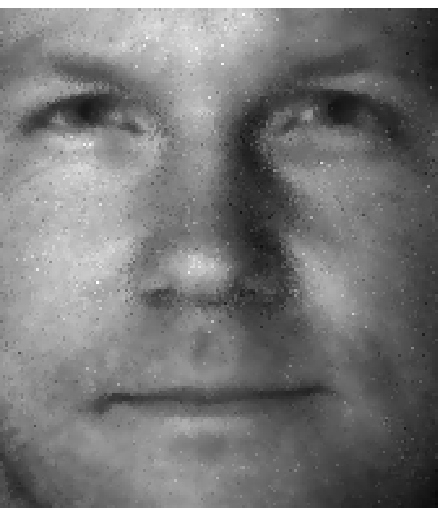}\label{fig_first_case}}
\subfigure{\includegraphics[width=0.16\linewidth]{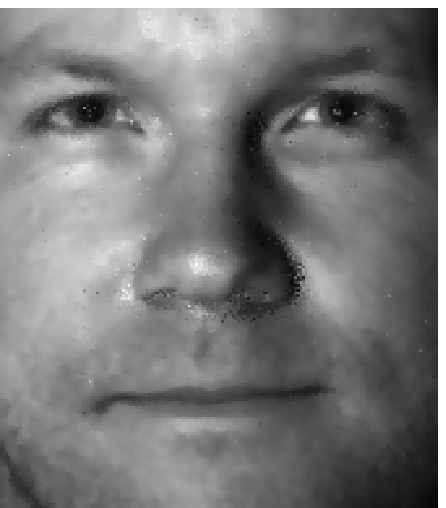}\label{fig_fourth_case}}
\subfigure{\includegraphics[width=0.16\linewidth]{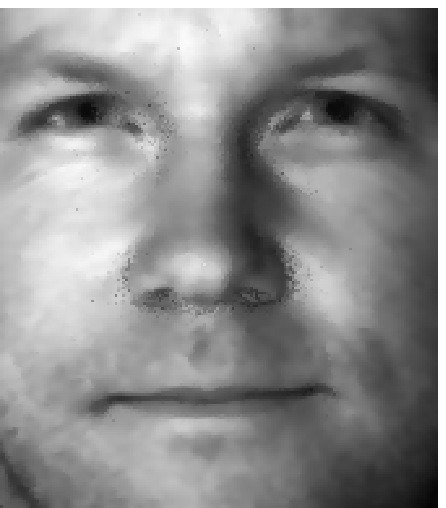}\label{fig_fourth_case}}\\
\caption{Face recovery results by these algorithms. From left column to right column: Input corrupted images (black pixels denote missing entries), original images, reconstruction results by PCP (1020.69sec), CWM (1830.18sec), RegL1 (2416.85sec) and RBF (52.73sec).}
\label{fig_5}
\end{figure}

Moreover, we implement a challenging problem to recover face images from incomplete line measurements. Considering the computational burden of the projection operator $\mathcal{P}_{Q}$, we resize the original images into $42\times48$ and normalize the raw pixel values to form data vectors of dimension 2016. Following \cite{wright:cpcp}, the input data is $\mathcal{P}_{Q}(D)$, where $Q$ is a subspace generated randomly with the dimension $0.75mn$.

Fig.\ 8 illustrates some reconstructed images by CPCP \cite{wright:cpcp} and RBF, respectively. It is clear that both CPCP and RBF effectively remove ``shadows" from faces images and simultaneously successfully recover both low-rank and sparse components from the reduced measurements.

\begin{figure}[t]
\centering
\subfigure{\includegraphics[width=0.18\linewidth]{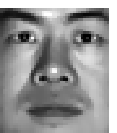}\label{fig_fourth_case}}
\subfigure{\includegraphics[width=0.18\linewidth]{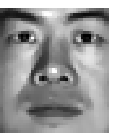}\label{fig_fifth_case}}
\subfigure{\includegraphics[width=0.18\linewidth]{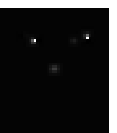}\label{fig_first_case}}
\subfigure{\includegraphics[width=0.18\linewidth]{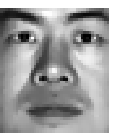}\label{fig_second_case}}
\subfigure{\includegraphics[width=0.18\linewidth]{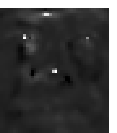}\label{fig_fourth_case}}
\subfigure{\includegraphics[width=0.18\linewidth]{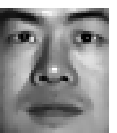}\label{fig_fourth_case}}
\subfigure{\includegraphics[width=0.18\linewidth]{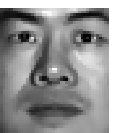}\label{fig_fifth_case}}
\subfigure{\includegraphics[width=0.18\linewidth]{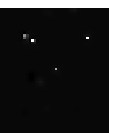}\label{fig_fourth_case}}
\subfigure{\includegraphics[width=0.18\linewidth]{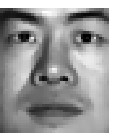}\label{fig_fourth_case}}
\subfigure{\includegraphics[width=0.18\linewidth]{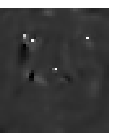}\label{fig_first_case}}\\
\caption{Face reconstruction results by CPCP and RBF, where the first column show the original images, the second and third columns show the low-rank and sparse components obtained by CPCP, while the last two columns show the low-rank and sparse components obtained by RBF.}
\label{fig_sim}
\end{figure}

\subsection{Collaborative Filtering}
Collaborative filtering is a technique used by some recommender systems \cite{liu:cf, lee:cf}. One of the main purposes is to predict the unknown preference of a user on a set of unrated items, according to other similar users or similar items. In order to evaluate our RBF method, some matrix completion experiments are conducted on three widely used recommendation system data sets: MovieLens100K with 100K ratings, MovieLens1M (ML-1M) with 1M ratings and MovieLens10M (ML-10M) with 10M ratings. We randomly split these data sets to training and testing sets such that the ratio of the training set to testing set is 9:1, and the experimental results are reported over 10 independent runs. We also compare our RBF method with APG\footnote[7]{\url{http://www.math.nus.edu.sg/~mattohkc/NNLS.html}} \cite{toh:apg}, Soft-Impute\footnote{\url{http://www.stat.columbia.edu/~rahulm/software.html}} \cite{mazumder:sr}, OptSpace\footnote{\url{http://web.engr.illinois.edu/~swoh/software/optspace/}} \cite{keshavan:mc} and LMaFit\footnote{\url{http://lmafit.blogs.rice.edu/}} \cite{wen:nsor}, and two state-of-the-art manifold optimization methods: ScGrass\footnote{\url{http://www-users.cs.umn.edu/~thango/}} \cite{ngo:gm} and RTRMC\footnote{\url{http://perso.uclouvain.be/nicolas.boumal/RTRMC/}} \cite{boumal:tr}. All other parameters are set to their default values for all compared algorithms. We use the Root Mean Squared Error (RMSE) as the evaluation measure, which is defined as
\begin{displaymath}
\textup{RMSE}=\sqrt{\frac{1}{|\Omega|}\Sigma_{(i,j)\in\Omega}(D_{ij}-L_{ij})^{2}},
\end{displaymath}
where $|\Omega|$ is the total number of ratings in the testing set, $L_{ij}$ denotes the ground-truth rating of user $i$ for item $j$, and $D_{ij}$ denotes the corresponding predicted rating.

The average RMSE on these three data sets is reported over 10 independent runs and is shown in Table 2. From the results shown in Table 2, we can see that, for some fixed ranks, most matrix factorization methods including ScGrass, RTRMC, LMaFit and our RBF method, except OptSpace, usually perform better than the two convex trace norm minimization methods, APG and Soft-Impute. Moreover, our bilinear factorization method with trace norm regularization consistently outperforms the other matrix factorization methods including OptSpace, ScGrass, RTRMC and LMaFit, and the two trace norm minimization methods, APG and Soft-Impute. This confirms that our robust bilinear factorization model with trace norm regularization is reasonable.

\begin{table*}[t]
\caption{RMSE of different methods on three data sets: MovieLens100K, MovieLens1M and MovieLens10M.}
\small
\centering
\begin{tabular}{l||ccc|ccc|ccc}
\hline
{Methods} & \multicolumn{3}{c|}{MovieLens100K} &\multicolumn{3}{c|}{MovieLens1M} & \multicolumn{3}{c}{MovieLens10M}\\
\hline
{APG}  &\multicolumn{3}{c|}{1.2142} &\multicolumn{3}{c|}{1.1528} &\multicolumn{3}{c}{0.8583}\\
{Soft-Impute} &\multicolumn{3}{c|}{1.0489} &\multicolumn{3}{c|}{0.9058} &\multicolumn{3}{c}{0.8615}\\
{OptSpace}  &\multicolumn{3}{c|}{0.9411} &\multicolumn{3}{c|}{0.9071} &\multicolumn{3}{c}{1.1357}\\
\hline
{Ranks}     & 5    & 6      & 7   & 5    & 6      & 7      & 5    & 6     & 7 \\
\hline
{ScGrass} &0.9647 &0.9809 &0.9945  &0.8847 &0.8852 &0.8936  &0.8359 &0.8290 &0.8247\\
{RTRMC}   &0.9837 &1.0617 &1.1642  &0.8875 &0.8893 &0.8960  &0.8463 &0.8442 &0.8386\\
{LMaFit}  &0.9468 &0.9540 &0.9568  &0.8918 &0.8920 &0.8853  &0.8576 &0.8530 &0.8423\\
{RBF}   &\textbf{0.9393} &\textbf{0.9513} &\textbf{0.9485} &\textbf{0.8672} &\textbf{0.8624} &\textbf{0.8591} &\textbf{0.8193} &\textbf{0.8159} &\textbf{0.8110}\\
\hline
\end{tabular}
\end{table*}

Furthermore, we also analyze the robustness of our RBF method with respect to its parameter changes: the given rank and the regularization parameter $\lambda$ on the MovieLens1M data set, as shown in Fig.\ 9, from which we can see that our RBF method is very robust against its parameter variations. For comparison, we also show the results of some related methods: ScGrass and LMaFit, OptSpace and RTRMC with varying ranks or different regularization parameters in Fig.\ 9. It is clear that, by increasing the number of the given ranks, the RMSE of ScGrass and LMaFit, RTRMC becomes dramatically increases, while that of our RBF method increase slightly. This further confirms that our bilinear matrix factorization model with trace norm regularization can significantly reduce the over-fitting problems of matrix factorization. ScGrass, RTRMC and OptSpace all have their spectral regularization models, respectively (for example, the formulation for OptSpace is $\min_{U, S, V}(1/2)\|\mathcal{P}_{\Omega}(USV^{T}-D)\|^{2}_{F}+\lambda\|S\|^{2}_{F}$.) We can see that our RBF method performs more robust than OptSpace, ScGrass and RTRMC in terms of the regularization parameter $\lambda$. Moreover, our RBF method is easily used to incorporate side-information as in \cite{shang:ssl, shang:grnmf, liu:nnr, li:pmf, yu:ahgl}.

\begin{figure}[t]
\centering
\subfigure[]{\includegraphics[width=0.46\linewidth]{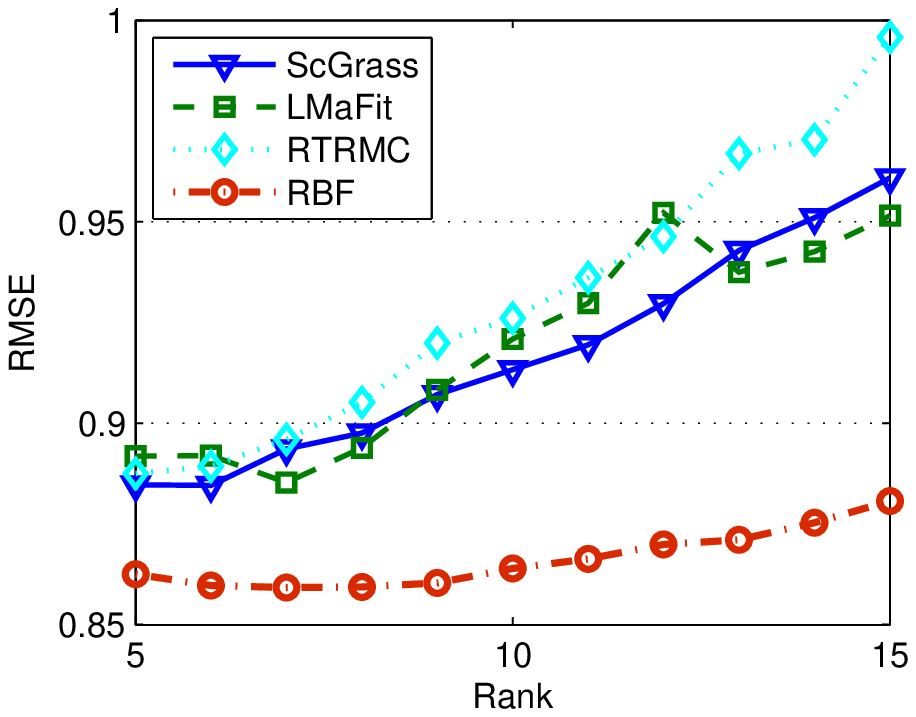}\label{fig_first_case}}
\subfigure[]{\includegraphics[width=0.46\linewidth]{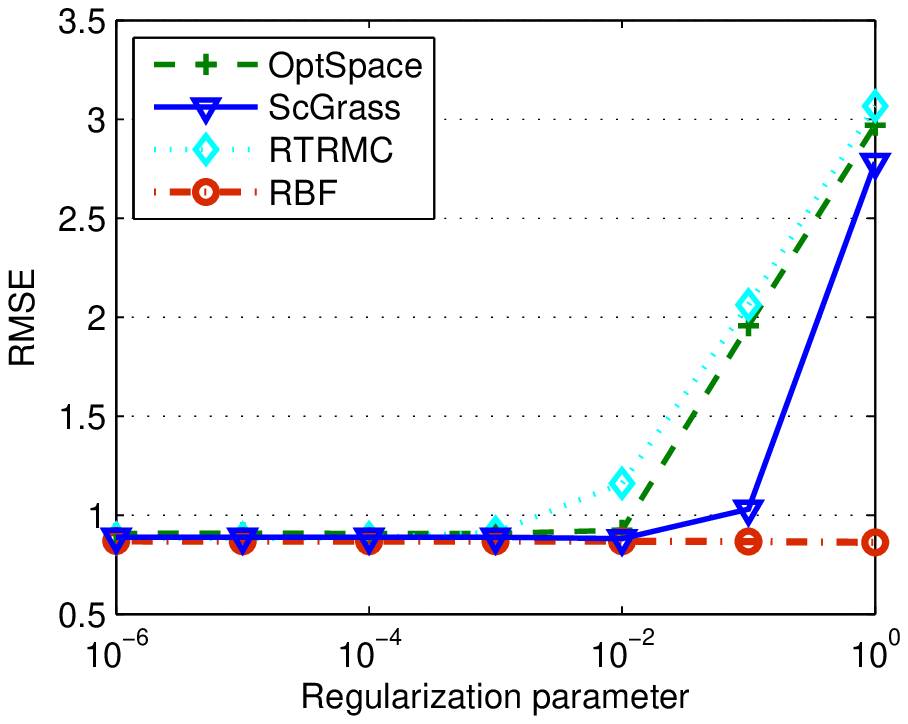}\label{fig_second_case}}
\caption{Results of our RBF method, ScGrass, LMaFit, and OptSpace against their parameters: (a) Rank and (b) Regularization parameter $\lambda$.}
\label{fig_sim}
\end{figure}

\section{CONCLUSIONS}
In this paper, we proposed a scalable robust bilinear structured factorization (RBF) framework for RMC and CPCP problems. Unlike existing robust low-rank matrix factorization methods, the proposed RBF method can not only address large-scale RMC problems, but can also solve low-rank and sparse matrix decomposition problems with incomplete or corrupted observations. To this end, we first presented two smaller-scale matrix trace norm regularized models for RMC and CPCP problems, respectively. Then we developed an efficient ADMM algorithm to solve both RMC and RPCA problems, and analyzed the suboptimality of the solution produced by our algorithm. Finally, we extended our algorithm to solve CPCP problems. Experimental results on real-world data sets demonstrated the superior performance of our RBF method in comparison with the state-of-the-art methods in terms of both efficiency and effectiveness.

\section*{APPENDIX A: Proof of Lemma 2}
\begin{proof}
Let $(L^{*},S^{*})$ be the optimal solution of (4), $g(L,\,S)=\|S\|_{1}+\lambda\|L\|_{*}$ and $\Gamma=\{(L,\,S)\,|\,\mathcal{P}_{\Omega}(D)=\mathcal{P}_{\Omega}(L+S)\}$, then we use contradiction to prove that $\mathcal{P}_{\Omega^{C}}(S^{*})=\textbf{0}$.

We first assume $\mathcal{P}_{\Omega^{C}}(S^{*})\neq \textbf{0}$. Let $\tilde{S}^{*}$ be $(\tilde{S}^{*})_{\Omega}=(S^{*})_{\Omega}$ and $(\tilde{S}^{*})_{\Omega^{C}}=0$, then we have $(L^{*},\,\tilde{S}^{*})\in \Gamma$ and $g(L^{*},\,\tilde{S}^{*})\leq g(L^{*},\,S^{*})$ that leads to a contradiction. Thus, $\mathcal{P}_{\Omega^{C}}(S^{*})=\textbf{0}$. Therefore, $(L^{*},\,S^{*})$ is also the optimal solution of (5).
\end{proof}

\section*{APPENDIX B: Proof of Lemma 3}
\begin{proof}
Let the SVD of $V^{T}$ be $V^{T}=\hat{U}\hat{\Sigma} \hat{V}^{T}$, then $UV^{T}=(U\hat{U})\hat{\Sigma} \hat{V}^{T}$. As $(U\hat{U})^{T}(U\hat{U})=I$, $(U\hat{U})\hat{\Sigma} \hat{V}^{T}$ is actually an SVD of $UV^{T}$. According to the definition of the trace norm, we have $\|V\|_{*}=$ $\|V^{T}\|_{*}=\textrm{tr}(\hat{\Sigma})=\|UV^{T}\|_{*}$.
\end{proof}

\section*{APPENDIX C: Proof of Theorem 4}
\begin{proof}
If we know that $(L^{\ast},\,S^{\ast})$ is a solution for the optimization problem (5), it is also a solution to
\begin{displaymath}
\begin{split}
&\min_{L,\,S,\,\textrm{rank}(L)=r}\,\|\mathcal{P}_{\Omega}(S)\|_{1}+\lambda\|L\|_{*},\\
&\;\textup{s.t.},\,\mathcal{P}_{\Omega}(D)=\mathcal{P}_{\Omega}(L+S),\,\mathcal{P}_{\Omega^{C}}(S)=\textbf{0}.
\end{split}
\end{displaymath}
Since for any $(L,\,S)$ with $\textrm{rank}(L)=r$, we can find $U\in \mathbb{R}^{m\times d}$ and $V\in \mathbb{R}^{n\times d}$ satisfying $UV^{T}=L$ and $\mathcal{P}_{\Omega}(D-UV^{T})=\mathcal{P}_{\Omega}(S)$, where $d\geq r$. Moreover, according to Lemma 3, we have
\begin{displaymath}
\begin{split}
&\min_{U,\,V,\,S}\,\|\mathcal{P}_{\Omega}(S)\|_{1}+\lambda\|V\|_{*}\\
&\;\textup{s.t.},\,\mathcal{P}_{\Omega}(D)=\mathcal{P}_{\Omega}(UV^{T}+S),\,U^{T}U=I,\\
=&\min_{U,\,V,\,S}\,\|\mathcal{P}_{\Omega}(S)\|_{1}+\lambda\|UV^{T}\|_{*}\\
&\;\textup{s.t.},\,\mathcal{P}_{\Omega}(D)=\mathcal{P}_{\Omega}(UV^{T}+S),\\
=&\min_{L,\,S,\,\textrm{rank}(L)=r}\,\|\mathcal{P}_{\Omega}(S)\|_{1}+\lambda\|L\|_{*}\\
&\;\textup{s.t.},\,\mathcal{P}_{\Omega}(D)=\mathcal{P}_{\Omega}(L+S),\\
\end{split}
\end{displaymath}
where $\mathcal{P}_{\Omega^{C}}(S)=$\textbf{0}. This completes the proof.
\end{proof}

\section*{APPENDIX D: Proof of Theorem 7}
\begin{proof}
We will prove the statement in Theorem 7 using mathematical induction.

\textbf{1.} While $k=1$, and following \cite{nick:mpp}, then the optimal solution to the problem (14) is given by
\begin{equation*}
\begin{split}
U^{*}_{1}=\widetilde{U}_{1}\widetilde{V}^{T}_{1},
\end{split}
\end{equation*}
where the skinny SVD of $P^{*}_{0}V^{*}_{0}$ is $P^{*}_{0}V^{*}_{0}=\widetilde{U}_{1}\widetilde{\Sigma}_{1}\widetilde{V}^{T}_{1}$.

By Algorithm 1, and with the same initial values, i.e., $U^{*}_{0}=U_{0}$, $V^{*}_{0}=V_{0}$, $S^{*}_{0}=S_{0}$ and $P^{*}_{0}=P_{0}$, then we have
\begin{equation*}
\begin{split}
U_{1}=Q,\qquad\textrm{QR}(P_{0}V_{0})=\textrm{QR}(P^{*}_{0}V^{*}_{0})=QR.
\end{split}
\end{equation*}
Hence, it can be easily verified that $\exists O_{1}\in \mathcal{N}$ satisfies $U^{*}_{1}=U_{1}O_{1}$, where $\mathcal {N}=\{A\in \mathbb{R}^{d\times d}, A^{T}A=I,\,AA^{T}=I\}$.

By the iteration step (19), we have
\begin{equation*}
\begin{split}
V^{*}_{1}&=\textrm{SVT}_{\lambda/\alpha_{k}}((P^{*}_{0})^{T}U^{*}_{1})\\
&=\textrm{SVT}_{\lambda/\alpha_{k}}((P^{*}_{0})^{T}U_{1}O_{1})\\
&=\textrm{SVT}_{\lambda/\alpha_{k}}((P^{*}_{0})^{T}U_{1})O_{1}\\
&=V_{1}O_{1}.
\end{split}
\end{equation*}
Thus, $ U^{*}_{1}(V^{*}_{1})^{T}=U_{1}V^{T}_{1}$. Furthermore, we have
\begin{equation*}
\begin{split}
S^{*}_{1}=S_{1},\,P^{*}_{1}=P_{1}\, \textrm{and}\, Y^{*}_{1}=Y_{1}.
\end{split}
\end{equation*}

\textbf{2.} While $k>1$, the result of Theorem 7 holds at the (\emph{k}-1)-th iteration, then following \cite{nick:mpp} and \cite{liu:as}, $U^{*}_{k}$ is updated by
\begin{equation*}
\begin{split}
&U^{*}_{k}=\widetilde{U}_{k}\widetilde{V}^{T}_{k},\\
\end{split}
\end{equation*}
where the skinny SVD of $P^{*}_{k-1}V^{*}_{k-1}$ is $P^{*}_{k-1}V^{*}_{k-1}=\widetilde{U}_{k}\widetilde{\Sigma}_{k}\widetilde{V}^{T}_{k}$.

By $P^{*}_{k-1}V^{*}_{k-1}=P^{*}_{k-1}V_{k-1}O_{k-1}$, and according to (15), then $\exists O_{k}\in \mathcal {N}$ satisfies $U^{*}_{k}=U_{k}O_{k}$. Furthermore, we have
\begin{equation*}
\begin{split}
U^{*}_{k}(V^{*}_{k})^{T}&=U^{*}_{k}\textrm{SVT}_{\lambda/\alpha_{k-1}}((U^{*}_{k})^{T}P^{*}_{k-1})\\
&=\textrm{SVT}_{\lambda/\alpha_{k}}(U^{*}_{k}(U^{*}_{k})^{T}P^{*}_{k-1})\\
&=\textrm{SVT}_{\lambda/\alpha_{k}}(U_{k}(U_{k})^{T}P_{k-1})\\
&=U_{k}V^{T}_{k},
\end{split}
\end{equation*}
and $V^{*}_{k}=V_{k}O_{k}$, $S^{*}_{k}=S_{k}$, $P^{*}_{k}=P_{k}$ and $Y^{*}_{k}=Y_{k}$.

Since $V^{*}_{k}=V_{k}O_{k}$, we also have $\|V^{*}_{k}\|_{*}=\|V_{k}\|_{*}$.

This completes the proof.
\end{proof}

\section*{APPENDIX E}
The proof sketch of Lemma 8 is similar to the one in \cite{liu:as}. We first prove that the boundedness of multipliers and some variables of Algorithm 1, and then analyze the convergence of Algorithm 1. To prove the boundedness, we first give the following lemmas.

\begin{lemma}
Let $\mathcal{X}$ be a real Hilbert space endowed with an inner product $\langle\cdot\rangle$ and a corresponding norm $\|\cdot\|$ (the nuclear norm or the $l_{1}$ norm), and $y\in\partial\|x\|$, where $\partial\|\cdot\|$ denotes the subgradient. Then $\| y\|^{*}=1$ if $x\neq 0$, and $\|y\|^{*}\leq 1$ if $x=0$, where $\|\cdot\|^{*}$ is the dual norm of the norm $\|\cdot\|$.
\end{lemma}

\begin{lemma}
Let $Y_{k+1}=Y_{k}+\alpha_{k}(D-U_{k+1}V^{T}_{k+1}-S_{k+1})$, $\widehat{Y}_{k+1}=Y_{k}+\alpha_{k}(D-U_{k+1}V^{T}_{k+1}-S_{k})$ \textrm{and} $\widetilde{Y}_{k+1}=Y_{k}+\alpha_{k}(D-U^{*}_{k+1}V^{T}_{k}-S_{k})$, where $U^{*}_{k+1}$ is the solution of the problem (14). Then the sequences $\{Y_{k}\}$, $\{\widehat{Y}_{k}\}$, $\{\widetilde{Y}_{k}\}$, $\{V_{k}\}$ and $\{S_{k}\}$ produced by Algorithm 1 are all bounded.
\end{lemma}

\begin{proof}
By the optimality condition of the problem (20) with respect to ${S_{k+1}}$, we have that
\begin{displaymath}
0\in \partial_{(S_{k+1})_{\Omega}}\mathcal{L}_{\alpha_{k}}(U_{k+1},V_{k+1},S_{k+1},Y_{k}),
\end{displaymath}
and
\begin{displaymath}
\mathcal{P}_{\Omega}(Y_{k}+\alpha_{k}(D-U_{k+1}V^{T}_{k+1}-S_{k+1}))\in \partial\|\mathcal{P}_{\Omega}(S_{k+1})\|_{1},
\end{displaymath}
i.e.,
\begin{equation}
\mathcal{P}_{\Omega}(Y_{k+1})\in \partial\|\mathcal{P}_{\Omega}(S_{k+1})\|_{1}.
\end{equation}

Furthermore, substituting $Y_{k+1}=Y_{k}+\alpha_{k}(D-U_{k+1}V^{T}_{k+1}-S_{k+1})$ into (23), we have
\begin{displaymath}
\mathcal{P}_{\Omega^{C}}(Y_{k+1})=\textbf{0}.
\end{displaymath}

By Lemma 12, we have
\begin{equation}
\|Y_{k+1}\|_{\infty}=\|\mathcal{P}_{\Omega}(Y_{k+1})\|_{\infty}\leq 1.
\end{equation}
Thus, the sequence $\{Y_{k}\}$ is bounded.

From the iteration procedure of Algorithm 1, we have that
\begin{equation*}
\begin{split}
&\mathcal{L}_{\alpha_{k}}(U_{k+1},V_{k+1},S_{k+1},Y_{k})\\
\leq&\mathcal{L}_{\alpha_{k}}(U_{k+1},V_{k+1},S_{k},Y_{k})\leq\mathcal{L}_{\alpha_{k}}(U_{k},V_{k},S_{k},Y_{k})\\ =&\mathcal{L}_{\alpha_{k-1}}(U_{k},V_{k},S_{k},Y_{k-1})+\beta_{k}\| Y_{k}-Y_{k-1}\|^2_{F}.\\
\end{split}
\end{equation*}
where $\beta_{k}=\frac{1}{2}\alpha^{-2}_{k-1}(\alpha_{k-1}+\alpha_{k})$ and $\alpha_{k}=\rho\alpha_{k-1}$.

Hence,
\begin{equation}
\sum^\infty_{k=1}2\alpha^{-2}_{k-1}(\alpha_{k-1}+\alpha_{k})=\frac{\rho(\rho+1)}{2\alpha_{0}(\rho-1)}<\infty.
\end{equation}
Thus, $\{\mathcal{L}_{\alpha_{k-1}}(U_{k},V_{k},S_{k},Y_{k-1})\}$ is upper bounded due to the boundedness of $\{Y_{k}\}$. Then
\begin{displaymath}
\begin{split}
&\lambda\|V_{k}\|_{*}+\|\mathcal{P}_{\Omega}(S_{k})\|_{1}\\
=&\mathcal{L}_{\alpha_{k-1}}(U_{k},V_{k},S_{k},Y_{k-1})-\frac{1}{2}\alpha^{-1}_{k-1}(\|Y_{k}\|^2_{F}-\|Y_{k-1}\|^2_{F}),
\end{split}
\end{displaymath}
is upper bounded, i.e., $\{V_{k}\}$ and $\{S_{k}\}$ are bounded, and $\{U_{k}V^{T}_{k}\}$  is also bounded.

We next prove that $\{\widetilde{Y}_{k}\}$ is bounded. Let $U^{*}_{k+1}$ denote the solution of the subproblem (14). By the optimality of $U^{*}_{k+1}$, then we have
\begin{displaymath}
\|Y_{k}+\alpha_{k}(D-U^{*}_{k+1}V^{T}_{k})\|^{2}_{F}\leq \|Y_{k}+\alpha_{k}(D-U_{k}V^{T}_{k}-S_{k})\|^{2}_{F},
\end{displaymath}
and by the definition of $\widetilde{Y}_{k}$, and $\alpha_{k+1}=\rho \alpha_{k}$, thus,
\begin{displaymath}
\|\widetilde{Y}_{k}\|^{2}_{F}\leq \|(1+\rho)Y_{k}-\rho Y_{k-1}\|^{2}_{F}.
\end{displaymath}
By the boundedness of ${V_{k}}$ and ${Y_{k}}$, then the sequence $\{\widetilde{Y}_{k}\}$ is bounded.

The optimality condition of the problem (16) with respect to ${V_{k+1}}$ is rewritten as follows:
\begin{equation}
U^{T}_{k+1}\widehat{Y}_{k+1}\in \lambda \partial\|V^{T}_{k+1}\|_{*}.
\end{equation}
By Lemma 12, we have that
\begin{displaymath}
\|U^{T}_{k+1}\widehat{Y}_{k+1}\|_{2}\leq \lambda.
\end{displaymath}
Thus, $U^{T}_{k+1}\widehat{Y}_{k+1}$ is bounded. Let $U^{\bot}_{k+1}$ denote the orthogonal complement of $U_{k+1}$, i.e., $U^{\bot}_{k+1}U_{k+1}=\textbf{0}$, and according to Theorem 7, then $\exists O_{k+1}$ satisfies $U^{*}_{k+1}=U_{k+1}O_{k+1}$, thus we have
\begin{displaymath}
\begin{split}
&(U^{\bot}_{k+1})^{T}\widehat{Y}_{k+1}\\
=&(U^{\bot}_{k+1})^{T}(Y_{k}+\alpha_{k}(D-U_{k+1}V^{T}_{k+1}-S_{k}))\\
=&(U^{\bot}_{k+1})^{T}(Y_{k}+\alpha_{k}(D-U_{k+1}O_{k+1}V^{T}_{k}-S_{k}))\\
=&(U^{\bot}_{k+1})^{T}(Y_{k}+\alpha_{k}(D-U^{*}_{k+1}V^{T}_{k}-S_{k}))\\
=&(U^{\bot}_{k+1})^{T}\widetilde{Y}_{k}.
\end{split}
\end{displaymath}
Thus, $\{(U^{\bot}_{k+1})^{T}\widehat{Y}_{k+1}\}$ is bounded due to the boundedness of $\{\widetilde{Y}_{k}\}$. Then we have
\begin{displaymath}
\|\widehat{Y}_{k+1}\|_{2}=\|U^{T}_{k+1}\widehat{Y}_{k+1}+(U^{\bot}_{k+1})^{T}\widehat{Y}_{k+1}\|_{2}\leq \|U^{T}_{k+1}\widehat{Y}_{k+1}\|_{2}+\|(U^{\bot}_{k+1})^{T}\widehat{Y}_{k+1}\|_{2}.
\end{displaymath}
Since both $U^{T}_{k+1}\widehat{Y}_{k+1}$ and $(U^{\bot}_{k+1})^{T}\widehat{Y}_{k+1}$ are bounded, the sequence $\{\widehat{Y}_{k}\}$ is bounded. This completes the proof.
\end{proof}

\textbf{\emph{Proof of Lemma 8}}:
\begin{proof}
$\textbf{1.}$ By $D-U_{k+1}V^{T}_{k+1}-S_{k+1}=\alpha^{-1}_{k}(Y_{k+1}-Y_{k})$, the boundedness of $\{Y_{k}\}$ and $\lim_{k\rightarrow\infty}\alpha_{k}=\infty$, we have that
\begin{displaymath}
\lim_{k\rightarrow\infty}D-U_{k+1}V^{T}_{k+1}-S_{k+1}=0.
\end{displaymath}

Thus, $(U_{k},V_{k},S_{k})$ approaches to a feasible solution.

$\textbf{2.}$ We prove that the sequences $\{S_{k}\}$\, and\, $\{U_{k}V^{T}_{k}\}$ are Cauchy sequences.

By $\|S_{k+1}-S_{k}\|=\alpha^{-1}_{k}\|Y_{k+1}-\widehat{Y}_{k}\|=o(\alpha^{-1}_{k})$ and
\begin{displaymath}
\sum^{\infty}_{k=1}\alpha^{-1}_{k-1}=\frac{\rho}{\alpha_{0}(\rho-1)}<\infty,
\end{displaymath}
thus, $\{S_{k}\}$ is a Cauchy sequence, and it has a limit, $S^{*}$.

Similarly, $\{U_{k}V^{T}_{k}\}$ is also a Cauchy sequence, therefore it has a limit, $U^{*}(V^{*})^{T}$.

This completes the proof.
\end{proof}

To prove Lemma 9, we first give the following lemma in \cite{liu:as}:
\begin{lemma}
Let $X$, $Y$ and $Q$ be matrices of compatible dimensions. If $Q$ obeys $Q^{T}Q=I$ and $Y\in\partial \|X\|_{*}$, then
\begin{displaymath}
QY\in \partial\|QX\|_{*}.
\end{displaymath}
\end{lemma}

\textbf{\emph{Proof of Lemma 9:}}
\begin{proof}
Let the skinny SVD of $P_{k}=D-S_{k}+Y_{k}/\alpha_{k}$ be $P_{k}=\widehat{U}_{k}\widehat{\Sigma}_{k}\widehat{V}^{T}_{k}$, then it can be calculated that
\begin{displaymath}
\begin{split}
\textrm{QR}(P_{k}V_{k})=\textrm{QR}(\widehat{U}_{k}\widehat{\Sigma}_{k}\widehat{V}^{T}_{k}V_{k}).
\end{split}
\end{displaymath}
Let the full SVD of $\widehat{\Sigma}_{k}\widehat{V}^{T}_{k}V_{k}$ be $\widehat{\Sigma}_{k}\widehat{V}^{T}_{k}V_{k}=\widetilde{U}\widetilde{\Sigma}\widetilde{V}^{T}$ (note that $\widetilde{U}$ and $\widetilde{V}$ are orthogonal matrices), then it can be calculated that
\begin{displaymath}
\begin{split}
\textrm{QR}(\widehat{U}_{k}\widehat{\Sigma}_{k}\widehat{V}^{T}_{k}V_{k})=\textrm{QR}(\widehat{U}_{k}\widetilde{U}\widetilde{\Sigma}\widetilde{V}^{T})=QR, \;U_{k+1}=Q.
\end{split}
\end{displaymath}

Then $\exists O$ and $O^{T}O=OO^{T}=I$ such that $U_{k+1}=\widehat{U}_{k}\widetilde{U}O$, which simply leads to
\begin{displaymath}
\begin{split}
U_{k+1}U^{T}_{k+1}=\widehat{U}_{k}\widetilde{U}OO^{T}\widetilde{U}^{T}\widehat{U}^{T}_{k}=\widehat{U}_{k}\widehat{U}^{T}_{k}.
\end{split}
\end{displaymath}
Hence,
\begin{displaymath}
\begin{split}
\widehat{Y}_{k+1}-U_{k+1}U^{T}_{k+1}\widehat{Y}_{k+1}&=\mu_{k}((D-S_{k}+Y_{k}/\mu_{k})-U_{k+1}U^{T}_{k+1}(D-S_{k}+Y_{k}/\mu_{k}))\\
&=\mu_{k}(\widehat{U}_{k}\widehat{\Sigma}_{k}\widehat{V}^{T}_{k}-U_{k+1}U^{T}_{k+1}\widehat{U}_{k}\widehat{\Sigma}_{k}\widehat{V}^{T}_{k})\\
&=\mu_{k}(\widehat{U}_{k}\widehat{\Sigma}_{k}\widehat{V}^{T}_{k}-\widehat{U}_{k}\widehat{U}^{T}_{k}\widehat{U}_{k}\widehat{\Sigma}_{k}\widehat{V}^{T}_{k})=0,
\end{split}
\end{displaymath}
i.e.,
\begin{displaymath}
\begin{split}
\widehat{Y}_{k+1}=U_{k+1}U^{T}_{k+1}\widehat{Y}_{k+1}.
\end{split}
\end{displaymath}
By (32) and Lemma 14, we have
\begin{displaymath}
\begin{split}
U_{k+1}U^{T}_{k+1}\widehat{Y}_{k+1}\in \lambda\partial\|U_{k+1}V^{T}_{k+1}\|_{*}.
\end{split}
\end{displaymath}
Thus, we have
\begin{displaymath}
\begin{split}
\widehat{Y}_{k+1}\in \lambda\partial\|U_{k+1}V^{T}_{k+1}\|_{*}\;\,\textrm{and}\;\, \mathcal{P}_{\Omega}(Y_{k+1})\in\partial\|\mathcal{P}_{\Omega}(S_{k+1})\|_{1},\forall k.
\end{split}
\end{displaymath}
Since the above conclusion holds for any $k$, it naturally holds at $(U^{*},V^{*},S^{*})$:
\begin{equation}
\begin{split}
\widehat{Y}^{*}=\widehat{Y}_{k^{*}+1}\in\lambda\partial\|U^{*}(V^{*})^{T}\|_{*} \:\textrm{and} \:\mathcal{P}_{\Omega}(Y^{*})=\mathcal{P}_{\Omega}(Y_{k^{*}+1})\in\partial\|\mathcal{P}_{\Omega}(S^{*})\|_{1}.
\end{split}
\end{equation}
Given any feasible solution $(U,V,S)$ to the problem (9), by the convexity of matrix norms and (33), and $\mathcal{P}_{\Omega^{C}}(Y^{*})=\textbf{0}$, it can be calculated that
\begin{displaymath}
\begin{split}
&\|\mathcal{P}_{\Omega}(S)\|_{1}+\lambda\|V\|_{*}=\|\mathcal{P}_{\Omega}(S)\|_{1}+\lambda\|UV^{T}\|_{*}\\
&\geq \|\mathcal{P}_{\Omega}(S^{*})\|_{1}+\langle \mathcal{P}_{\Omega}(Y^{*}),\mathcal{P}_{\Omega}(S-S^{*})\rangle+
\lambda\|U^{*}(V^{*})^{T}\|_{*}+\langle \widehat{Y}^{*},UV^{T}-U^{*}(V^{*})^{T}\rangle\\
&=\|\mathcal{P}_{\Omega}(S^{*})\|_{1}+\langle \mathcal{P}_{\Omega}(Y^{*}),S-S^{*}\rangle+\lambda\|U^{*}(V^{*})^{T}\|_{*}+\langle \widehat{Y}^{*},UV^{T}-U^{*}(V^{*})^{T}\rangle\\
&=\|\mathcal{P}_{\Omega}(S^{*})\|_{1}+\lambda\|U^{*}(V^{*})^{T}\|_{*}+\langle \mathcal{P}_{\Omega}(Y^{*}),UV^{T}+S-U^{*}(V^{*})^{T}-S^{*}\rangle+\langle \widehat{Y}^{*}-\mathcal{P}_{\Omega}(Y^{*}),UV^{T}-U^{*}(V^{*})^{T}\rangle\\
&=\|\mathcal{P}_{\Omega}(S^{*})\|_{1}+\lambda\|U^{*}(V^{*})^{T}\|_{*}+\langle \mathcal{P}_{\Omega}(Y^{*}),UV^{T}+S-U^{*}(V^{*})^{T}-S^{*}\rangle+\langle \widehat{Y}^{*}-Y^{*},UV^{T}-U^{*}(V^{*})^{T}\rangle.
\end{split}
\end{displaymath}

By Lemma 8 and $\|\mathcal{P}_{\Omega}(Y^{*})\|_{\infty}\leq 1$, we have that $\|UV^{T}+S-U^{*}(V^{*})^{T}-S^{*}\|_{\infty}= \| D-U^{*}(V^{*})^{T}-S^{*}\|_{\infty}\leq\varepsilon$, which leads to
\begin{displaymath}
\begin{split}
|\langle \mathcal{P}_{\Omega}(Y^{*}),UV^{T}+S-U^{*}(V^{*})^{T}-S^{*}\rangle| &\leq \| \mathcal{P}_{\Omega}(Y^{*})\|_{\infty}\| UV^{T}+S-U^{*}(V^{*})^{T}-S^{*}\|_{1}\\
&=\|\mathcal{P}_{\Omega}(Y^{*})\|_{\infty}\|D-U^{*}(V^{*})^{T}-S^{*}\|_{1}\\
&\leq mn\|D-U^{*}(V^{*})^{T}-S^{*}\|_{\infty}\leq mn\varepsilon.
\end{split}
\end{displaymath}
Hence,
\begin{displaymath}
\begin{split}
\|\mathcal{P}_{\Omega}(S)\|_{1}+\lambda\|V\|_{*}\geq\|\mathcal{P}_{\Omega}(S^{*})\|_{1}+\lambda\|V^{*}\|_{*}+\langle \widehat{Y}^{*}-Y^{*},UV^{T}-U^{*}(V^{*})^{T}\rangle-mn\varepsilon.
\end{split}
\end{displaymath}
This completes the proof.
\end{proof}

\textbf{\emph{Proof of Theorem 10:}}
\begin{proof}
It is worth nothing that $(U,\,V=\textbf{0},\,S=D)$ is feasible to (9). Let $(U^{g},V^{g},S^{g})$ be a globally optimal solution to (9), then we have
\begin{displaymath}
\begin{split}
\lambda\|V^{g}\|_{*}\leq \|\mathcal{P}_{\Omega}(S^{g})\|_{1}+ \lambda\|V^{g}\|_{*}\leq \|D\|_{1}=\|\mathcal{P}_{\Omega}(D)\|_{1}.
\end{split}
\end{displaymath}
By the proof procedure of Lemma 13 and $\alpha_{0}=\frac{1}{\|\mathcal{P}_{\Omega}(D)\|_{F}}$, we have that $V^{*}$ is bounded by
\begin{displaymath}
\begin{split}
\lambda\|V^{*}\|_{*}&\leq \|\mathcal{P}_{\Omega}(S^{*})\|_{1}+ \lambda\|V^{*}\|_{*}\\
&\leq \mathcal{L}_{\alpha_{k^{*}}}(U_{k^{*}+1},V_{k^{*}+1},S_{k^{*}+1},Y_{k^{*}})+\frac{\|Y_{k^{*}}\|^{2}_{F}}{2\mu_{k^{*}}}\\
&\leq \frac{mn}{\alpha_{0}}(\frac{\rho(1+\rho)}{\rho-1}+\frac{1}{2\rho^{k^{*}}})\\
&=mn\|\mathcal{P}_{\Omega}(D)\|_{F}(\frac{\rho(1+\rho)}{\rho-1}+\frac{1}{2\rho^{k^{*}}}).
\end{split}
\end{displaymath}
Thus,
\begin{equation}
\begin{split}
\|U^{g}(V^{g})^{T}-U^{*}(V^{*})^{T}\|_{*}\leq \|V^{g}\|_{*}+\|V^{*}\|_{*}\leq c_{1}.
\end{split}
\end{equation}
Note that $|\langle M,N\rangle|\leq \|M\|_{2}\|N\|_{*}$ (please see \cite{tomioka:ctd}) holds for any matrices $M$ and $N$. By Lemma 9 and (34), we have
\begin{displaymath}
\begin{split}
f^{g}&=\|\mathcal{P}_{\Omega}(S^{g})\|_{1}+\lambda\|V^{g}\|_{*}\\
&\geq \|\mathcal{P}_{\Omega}(S^{*})\|_{1}+\lambda\|V^{*}\|_{*}+\langle \widehat{Y}^{*}-Y^{*},U^{g}(V^{g})^{T}-U^{*}(V^{*})^{T}\rangle-mn\varepsilon\\
&\geq f^{*}- \|\widehat{Y}^{*}-Y^{*}\|_{2}\|U^{g}(V^{g})^{T}-U^{*}(V^{*})^{T}\|_{*}-mn\varepsilon\\
&= f^{*}- \varepsilon_{1}\|U^{g}(V^{g})^{T}-U^{*}(V^{*})^{T}\|_{*}-mn\varepsilon\\
&\geq f^{*}-c_{1}\varepsilon_{1}-mn\varepsilon.
\end{split}
\end{displaymath}
This completes the proof.
\end{proof}

\textbf{\emph{Proof of Theorem 11:}}
\begin{proof}
Let $L=U^{*}(V^{*})^{T}$ and $S=S^{*}$, then $(L,\,S)$ is a feasible solution to the RMC problem (5). By the convexity of the problem (5) and the optimality of $(L^{0},S^{0})$, it naturally follows that
\begin{displaymath}
f^{0}\leq f^{*}.
\end{displaymath}
Let $L^{0}=U^{0}\Sigma^{0}(V^{0})^{T}$ be the skinny SVD of $L^{0}$. Construct $U^{'}=U^{0}$, $(V^{'})^{T}=\Sigma^{0}(V^{0})^{T}$ and $S^{'}=S^{0}$. When $d\geq r$, we have
\begin{displaymath}
D=L^{0}+S^{0}=U^{0}\Sigma^{0}(V^{0})^{T}+S^{0}=U^{'}(V^{'})^{T}+S^{'},
\end{displaymath}
i.e., $(U^{'},V^{'},S^{'})$ is a feasible solution to the problem (9). By Theorem 10, it can be concluded that
\begin{displaymath}
\begin{split}
f^{*}-c_{1}\varepsilon_{1}-mn\varepsilon\leq \lambda\|V^{'}\|_{*}+\|\mathcal{P}_{\Omega}(S^{'})\|_{1}=\lambda\|\Sigma^{0}\|_{*}+\|\mathcal{P}_{\Omega}(S^{0})\|_{1}=f^{0}.
\end{split}
\end{displaymath}

For $d\leq r$, we decompose the skinny SVD of $L^{0}$ as
\begin{displaymath}
\begin{split}
L^{0}=U_{0}\Sigma_{0}V^{T}_{0}+U_{1}\Sigma_{1}V^{T}_{1},
\end{split}
\end{displaymath}
where $U_{0},\,V_{0}$ (resp.\ $U_{1},\,V_{1}$) are the singular vectors associated with the $d$ largest singular values (resp.\ the rest singular values smaller than or equal to $\sigma_{d}$). With these notations, we have a feasible solution to the problem (9) by constructing
\begin{displaymath}
\begin{split}
U^{''}=U_{0},\;(V^{''})^{T}=\Sigma_{0}V^{T}_{0}\,\;\textrm{and}\,\;S^{''}=D-U_{0}\Sigma_{0}V^{T}_{0}=S^{0}+U_{1}\Sigma_{1}V^{T}_{1}.
\end{split}
\end{displaymath}
By Theorem 10, it can be calculated that
\begin{displaymath}
\begin{split}
f^{*}-c_{1}\varepsilon_{1}-mn\varepsilon &\leq f^{g}\leq \lambda\|V^{''}\|_{*}+\|\mathcal{P}_{\Omega}(S^{''})\|_{1}\\
&\leq\lambda\|\Sigma_{0}\|_{*}+\|\mathcal{P}_{\Omega}(S^{o}+U_{1}\Sigma_{1}V^{T}_{1})\|_{1}\\
&\leq\lambda\|L^{0}\|_{*}-\lambda\|\Sigma_{1}\|_{*}+\|\mathcal{P}_{\Omega}(S^{0})\|_{1}+\|\mathcal{P}_{\Omega}(U_{1}\Sigma_{1}V^{T}_{1})\|_{1}\\
&\leq f^{0}-\lambda\|\Sigma_{1}\|_{*}+\|U_{1}\Sigma_{1}V^{T}_{1}\|_{1}\\
&\leq f^{0}-\lambda\|\Sigma_{1}\|_{*}+\sqrt{mn}\|U_{1}\Sigma_{1}V^{T}_{1}\|_{F}\\
&\leq f^{0}-\lambda\|\Sigma_{1}\|_{*}+\sqrt{mn}\|U_{1}\Sigma_{1}V^{T}_{1}\|_{*}\\
&\leq f^{0}+(\sqrt{mn}-\lambda)\|\Sigma_{1}\|_{*}\\
&\leq f^{0}+(\sqrt{mn}-\lambda)\sigma_{d+1}(r-d).
\end{split}
\end{displaymath}
This completes the proof.
\end{proof}

\bibliographystyle{elsarticle-num}
\bibliography{your-bib-database}

\end{document}